\documentclass[journal,twoside,web]{ieeecolor}

\makeatletter

\let\proof\@undefined
\let\endproof\@undefined
\makeatother

\usepackage{generic}
\usepackage[sort,compress]{cite}
\usepackage{optidef} 

\usepackage{amsmath,mathrsfs} 
\usepackage{amssymb}  
\usepackage{bbm}
\usepackage{booktabs}
\usepackage[ruled]{algorithm2e}
\usepackage{algorithmic}
\usepackage{breqn}
\usepackage{multicol}
\usepackage{diagbox}
\usepackage[normalem]{ulem} 

\usepackage{amsthm}
\theoremstyle{definition}
\newtheorem{definition}{Definition}
\newtheorem{theorem}{Theorem}
\newtheorem{problem}{Problem}
\newtheorem{lemma}[]{Lemma}

\mathchardef\mhyphen="2D 

\usepackage{hyperref}
\usepackage{booktabs}
\usepackage[flushleft]{threeparttable}

\usepackage{scalerel,stackengine}
\newcommand\pig[1]{\scalerel*[5pt]{\big#1}{%
		\ensurestackMath{\addstackgap[1.5pt]{\big#1}}}}

\definecolor{mypink}{rgb}{0.858, 0.188, 0.478}
\definecolor{mygreen}{rgb}{0.2, 0.7, 0.2}
\definecolor{mycol}{rgb}{0.8, 0.5, 0.5}
\definecolor{mygray}{rgb}{0.5, 0.5, 0.5}

\newcommand{\blu}[1]{{\color{black} #1}}

\newcommand{\gray}[1]{{\color{mygray}#1}}

\def\verticaldistance{5pt}
\def\lessverticaldistance{2pt}

\newcommand{\real}{\mathbb{R}}
\newcolumntype{C}{>{$}c<{$}} 

\definecolor{subsectioncolor}{rgb}{0,0.541,0.855}

\begin{document}
	
	\title{Scalable Operator Allocation for Multi-Robot Assistance: A Restless Bandit Approach}
	\author{Abhinav Dahiya, Nima Akbarzadeh, Aditya Mahajan, Stephen L. Smith
		\thanks{This research is supported in part by the Natural Sciences and Engineering Research Council of Canada (NSERC) and in part by the Innovation for Defence Excellence and Security (IDEaS) Program of the Canadian Department of National Defence through grant CFPMN2-037, and Fonds de Recherche du Quebec-Nature et technologies (FRQNT).} 
		\thanks{Abhinav Dahiya and Stephen L. Smith are with Department of Electrical and Computer Engineering, University of Waterloo, Waterloo
			(\href{mailto:abhinav.dahiya@uwaterloo.ca}{abhinav.dahiya@uwaterloo.ca}, \href{mailto:stephen.smith@uwaterloo.ca}{stephen.smith@uwaterloo.ca})}%
		\thanks{Nima Akbarzadeh and Aditya Mahajan are with Department of Electrical and Computer Engineering, McGill University, Montreal.
			(\href{mailto:nima.akbarzadeh@mail.mcgill.ca}{nima.akbarzadeh@mail.mcgill.ca}, \href{mailto:aditya.mahajan@mcgill.ca}{aditya.mahajan@mcgill.ca})}
	}
	
	\maketitle
	
	\begin{abstract}
		In this paper, we consider the problem of allocating human operators in a system with multiple semi-autonomous robots. \blu{Each robot is required to perform an independent sequence of tasks, subjected to a chance of failing and getting stuck in a fault state at every task.} If and when required, a human operator can assist or teleoperate a robot. Conventional MDP techniques used to solve such problems face scalability issues due to exponential growth of state and action spaces with the number of robots and operators. 
		\blu{In this paper we derive conditions under which the operator allocation problem is indexable, enabling the use of the Whittle index heuristic. The conditions can be easily checked to verify indexability, and we show that they hold for a wide range of problems of interest. Our key insight is to leverage the structure of the value function of individual robots, resulting in conditions that can be verified separately for each state of each robot.}
		\blu{We apply these conditions to two types of transitions commonly seen in remote robot supervision systems. Through numerical simulations, we demonstrate the efficacy of Whittle index policy as a near-optimal and scalable approach that outperforms existing scalable methods.}
	\end{abstract}
	
	
	\section{Introduction}
	\label{sec:intro}
	Advances in robot autonomy have lead to a decrease in the necessity of strict human supervision of robots.
	This has enabled the development of human-robot collaborative systems where the task is primarily executed by a number of semi-autonomous robots requiring intermittent assistance from a human teammate, either in event of a \textit{fault} \cite{wang2014human, swamy2020scaled} or to further increase performance of the multi-agent teams in warehouse operation \cite{rosenfeld2017intelligent}, search-and-rescue \cite{khasawneh2019human} or in a social setting \cite{zheng2013supervisory}.
	However, identifying which robot to assist in an uncertain environment is a challenging task for human operators \cite{chen2012supervisory, rosenfeld2017intelligent}. Moreover, as the number of robots increases, it becomes challenging for the operators to maintain awareness of every robot, which cripples system's performance \cite{olsen2004fan, chien2013imperfect}.
	Therefore, human operators can benefit from having a decision support system (DSS) that advises which robots require attention and when \cite{chen2014human, rosenfeld2017intelligent}.
	
	In this paper, we present such a DSS for a multi-robot system comprising a fleet of semi-autonomous robots with multiple human operators available for assistance if and when required. Figure~\ref{fig:main_setup} presents an overview of the problem setup, showing $K$ robots navigating in a city block-like environment, moving from start to goal locations. While navigating, a robot passes through a series of waypoints, each characterized by a different probability of success in progressing to the next waypoint. There is also a possibility that the robot may fail at a task and get stuck in a fault (error) state from which assistance from a human operator is required to continue. There are $M$ identical human operators available ($M\leq K$), each of whom can assist/teleoperate at most one robot at a time.
	While being assisted by an operator, robots have different probabilities of success and failure, and to get out of a fault state.
	\begin{figure}
		\centering
		\includegraphics[width=0.8\columnwidth]{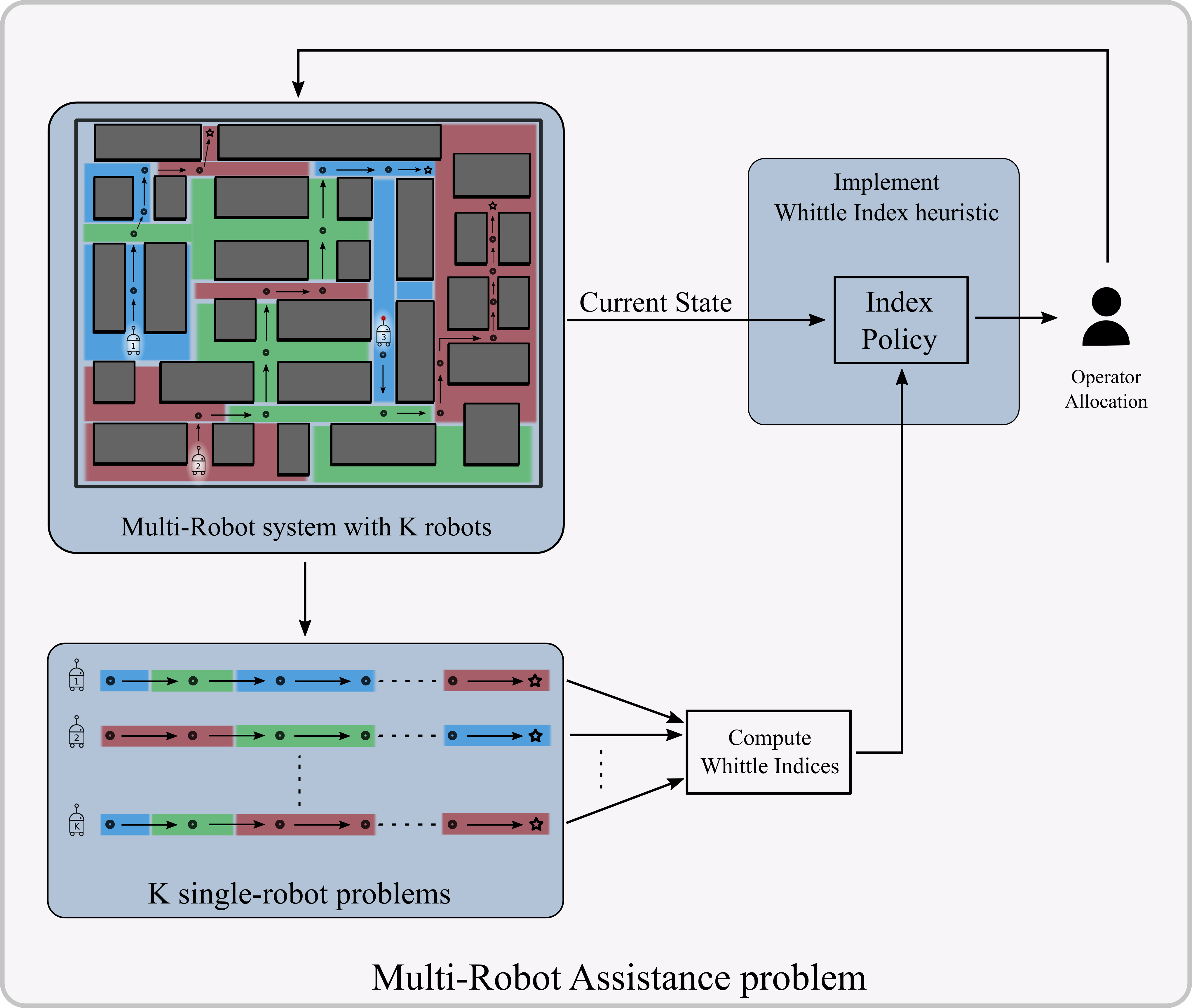}
		\caption{Overview of the multi-robot assistance problem for robots navigating in an environment. A number of mobile robots are tasked to navigate through a series of waypoints. 
			Operators are allocated to the robots when required. This is done by separating the complete $K$-robot problem into $K$ single robot problems and computing the Whittle index heuristic. Given the current state of the system, this heuristic can be used to efficiently compute the operator allocation.}
		\label{fig:main_setup}
	\end{figure}
	%
	It is possible to solve the above problem by modelling it as a Markov Decision Process (MDP) \cite{puterman2014markov}. However, such formulations suffer from the curse of dimensionality, making conventional MDP-based solution techniques scale poorly with problem size~\mbox{\cite{littman2013complexity}}. 
	Moreover, the policy needs to be re-solved every time a robot or an operator enters or leaves the system.
	
	\blu{The most closely related work to our problem is presented in \cite{rosenfeld2017intelligent}, where the authors discuss single-operator multi-robot supervision systems. 
		An advising agent guides the operator on which robot they should assist. 
		The problem is solved using an $l$-step look-ahead (myopic) approach, which provides an efficient and practical solution, but suffers from scalability issues with increasing number of operators and the look-ahead steps.
		Researchers have also discussed deterministic versions of the problem, where exact outcomes of robots' actions and times for fault occurrences are known, and the allocation policy is determined using a sampling-based planner \cite{zanlongo2021scheduling}. In \cite{hari2020approximation}, the authors present an approximation algorithm for a similar scheduling problem. These approaches however do not apply to the stochastic problems.
		
		In this paper, we show how an index-based policy can provide a scalable and better performing solution than the myopic approach for the given multi-robot multi-operator allocation problem under stochastic transitions.}
	\blu{Specifically, our work makes the following contributions:
		
		1) We show that the operator allocation problem with multiple independent robots can be formulated as an instance of the Restless Multi-Armed Bandit (RMAB) problem. 
		We leverage this formulation to obtain a Whittle index-based policy that scales linearly with the number of robots and is independent of the number of operators.
		
		2) We derive simple conditions to verify indexability of the model.
		These conditions can be checked independently for each state of each robot, thus providing a method that scales linearly with the size of the problem. 
		This method can be applied to systems with any number of states and does not require the optimal policy to be of a threshold type.
		
		3) We then implement our approach in two practical scenarios and present numerical experiments. The results show that the proposed method provides near-optimal solutions and outperforms existing efficient solution approaches, namely the reactive policy, $1$- and $2$-step myopic policies, and the {benefit-maximizing} policy. 
	}

	\subsection{Background and Related Work}
	\label{sec:background}
	%
	%
	\blu{The problem of allocating operators in a multi-robot team bares similarities with the disciplines of multi-robot supervision, task scheduling and queuing theory. In this section, we briefly review the related research, followed by an introduction of a Restless-Multi Armed Bandits.
		
		In the literature, several studies discuss the problem of enabling human operators to assist multiple robots such as a team of navigating robots, a fleet of multiple UAVs, or a team performing search and rescue operations~\cite{peters2015human, crandall2010computing}.
		To understand and improve human supervision, researchers have used frameworks such as sliding autonomy to incorporate various human-robot team capabilities (like coordination and situational awareness)~\cite{music2017control, dias2008sliding}. Some studies also present interaction interfaces to facilitate and improve such supervision \cite{szafir2017designing, kirchner2016intelligent}.}
	The problem of assisting a number of independent robots, has also been studied under a learning framework. 
	\blu{The approach presented in \cite{swamy2020scaled} learns the decision-making model of human operator from recorded data 
		and tries to replicate that behaviour, optimizing based on the operator's internal utility function.} 
	In contrast, the problem presented in this paper is designed to optimize a global performance metric assuming the knowledge of success and failure rates of robots with and without an operator allocated to them. 
	\blu{Such knowledge can be estimated using recorded data similar to the work presented in \cite{rosenfeld2017intelligent}. For the scope of this paper, we will assume this knowledge takes the form of known transition probabilities.}
	
	%
	
	\blu{In the queuing discipline, several studies have investigated the effects of different queuing techniques \cite{chien2018attention} or threshold-based strategies \cite{raeissi2017balking} to prioritize operator's attention to the robots. 
		However, the model that we study is different from a queuing model as it is possible for the robots to complete their tasks without the help of operators, and for the operators to be allocated to robots not stuck in a fault state.}


	
	\blu{The multi-target--multi-agent problems form another class of problems similar to the operator allocation problem. These problems deal with allocation of multiple agents to a number of targets aiming to detect or follow the targets under certain constraints \cite{berger2014new, pei2016multi}. However, our problem setup is different because the behaviour of the targets (robots) changes with the allocation of agents (operators) and it is not possible to allocate multiple agents to a single target at once. Moreover, our problem presents a collaborative task, where both the robots and operators are working to achieve a common goal.}
	
	\subsubsection*{Restless Multi-Armed Bandit}
	Restless Multi-Armed Bandits (RMAB), first introduced in \cite{whittleIndex}, is a generalization of Multi-Armed Bandits (MAB) \cite{gittins1979bandit} which has been previously used in problems like assisting human partners \cite{chan2019assistive} and distributing resources among human teammates \cite{claure2020multi}. RMAB is a class of scheduling problems where limited resources have to be allocated among several alternative choices. Each choice, referred to as an \textit{arm}, is a discrete-time controlled Markov process which evolves depending on the resource allocated to it. 
	\blu{RMAB framework has been applied to problems in stochastic scheduling, patrol planning, sensor management and resource allocation in general~\mbox{\cite{mahajan2008multi}}.} 
	
	
	Finding the optimal policy for RMAB suffers from the curse of dimensionality as the state space grows exponentially with the number of arms. In general, obtaining the optimal policy in an RMAB is {PSPACE-hard} \cite{tsitsiklis1999complexity}. 
	However, the \textit{Whittle index policy} offers a simpler and scalable alternative to the optimal policy.
	%
	\blu{Even though the Whittle index policy does not guarantee an optimal solution, it minimizes expected cost for a relaxed problem under time-averaged constraint \cite{whittleIndex}.} 
	\blu{This approach is shown to work quite well for several scheduling and resource allocation problems \cite{xiong2021learning, borkar2017opportunistic, akbarzadeh2021conditions}. A few studies have also implemented index-based methods to solve a sensor scheduling problem \cite{wu2020optimal} or to serve a number of users transmitting a queue of data packets through a channel \cite{borkar2017opportunistic}.}
	Therefore, it is a reasonable approach to solve an RMAB given that the problem satisfies a technical condition known as \textit{Indexability} (more details in Section~\ref{sec:restless}).
	\blu{Unfortunately, it is difficult to verify this condition in general and there is no universal framework that applies to all problems.
		Existing methods proposed for verifying indexability have been investigated for specific systems such as two state restless bandits \cite{avrachenkov2013congestion, liu2010indexability} or restless bandits with optimal threshold-based policy \cite{avrachenkov2013congestion, akbarzadeh2019dynamic, akbarzadeh2020restless}.}

	\subsection{Organization}
	\color{black}
	The contents of this paper are organized as follows: The multi-robot assistance problem is presented in Section~\ref{sec:general_formulation}. We discuss the general Restless Bandit Problem and define property of indexability in Section \ref{sec:restless}. In Section \ref{sec:example_problems}, we present two practical classes of transition functions and establish conditions under which problem indexability is ensured. In Section~\ref{sec:index_policy}, we cover the calculation of Whittle index heuristic and present an efficient policy for the problem. Next, we present simulations of the problem in Section~\ref{sec:simulation} to examine validity and performance of the presented policy. The paper ends with a brief discussion and conclusion.
	
	\section{Multi-Robot Assistance Problem}
	\label{sec:general_formulation}
	Consider a decision support system (DSS), consisting of a team of $M$ human operators supervising a fleet of $K$ semi-autonomous robots. Each robot $k\in {\cal K} \coloneqq \{1,\ldots, K\}$ is required to complete a sequence of $N^k$ tasks to reach its goal. 
	We will use a fleet of robots delivering packages in a city as a running example but similar interpretations hold for other applications mentioned in previous sections \blu{(e.g., robots reaching a sequence of configurations \cite{zanlongo2021scheduling})}. In this case, the robot's trajectory would correspond to a series of waypoints that a robot needs to navigate to reach its destination (goal location).
	At each waypoint, a robot can either operate autonomously or be teleoperated by one of the human operators. We assume that all human operators are identical in the way they operate the robots and that a human operator can operate at most one robot at a time.
	We now provide a mathematical model for different components of the system\footnote{\blu{\textbf{Remark on notation:} Throughout this paper, we use calligraphic font to denote sets and roman font to denote variables. Uppercase letters are used to represent random variables and the corresponding lowercase letters represent their realizations. Bold letters are used for variables pertaining to multi-robot system while light letters represent corresponding single-robot variables.}}.
	\subsection{Model of the robots}
	It is assumed that when operating autonomously, each robot uses a pre-specified control algorithm to complete its task. For the delivery robot example, this could be, for instance, a SLAM-based local path planner that the robot uses for navigating between the waypoints. We will not model the details of this control algorithm but simply assume that this control is imperfect and occasionally causes the robot to enter a fault state while doing a task (e.g., delivery robot getting stuck in a pothole or losing its localization). 
	We model this behaviour by assuming that while completing each task, the robot may be in one of the two internal states: a \emph{normal} state (denoted by $s=0$) or a \emph{fault} state (denoted by $s=1$).  When a robot is being teleoperated, it may still be possible for it to enter into a fault state. 
	
	The operating state of robot $k \in \mathcal{K}$ at time~$t$, denoted by $x^k_t = (n^k_t, s^k_t)$, is tuple of its current task and internal state. The state space for robot $k$ is given by 
	\[
	\mathcal{X}^k \coloneqq \bigcup_{n=1}^{N^k} \{ (n,0), (n,1) \} \cup \{ (G,0) \},
	\]
	where the terminal state $(G,0)$ indicates that all tasks have been completed. The state space for all robots is denoted by $\pmb{\mathcal{X}} = \mathcal X^1 \times \cdots \times \mathcal X^K$.
	
	The state of a robot evolves differently depending on whether it is operating autonomously (denoted by mode $a^k = 0$) or teleoperated (denoted by $a^k=1$). Given robot $k \in \mathcal{K}$ in state $(n,s) \in \mathcal{X}^k$ operating in mode $a \in \{0,1\}$, let $p^{ka}_{ns}$ denote the probability of successfully completing the current task at the current time step and let $q^{ka}_{ns}$ denote the probability of toggling the current internal state (i.e. going from normal to fault state and vice-versa). A diagram describing these transitions is shown in Fig.~\ref{fig:state_transition}. Note that the terminal state $(G,0)$ is an absorbing state, so $p^{ka}_{G0} = 0$ and $q^{ka}_{G0} = 0$. 
	
	\begin{figure}[ht]
		\centering
		\includegraphics[width=0.75\columnwidth]{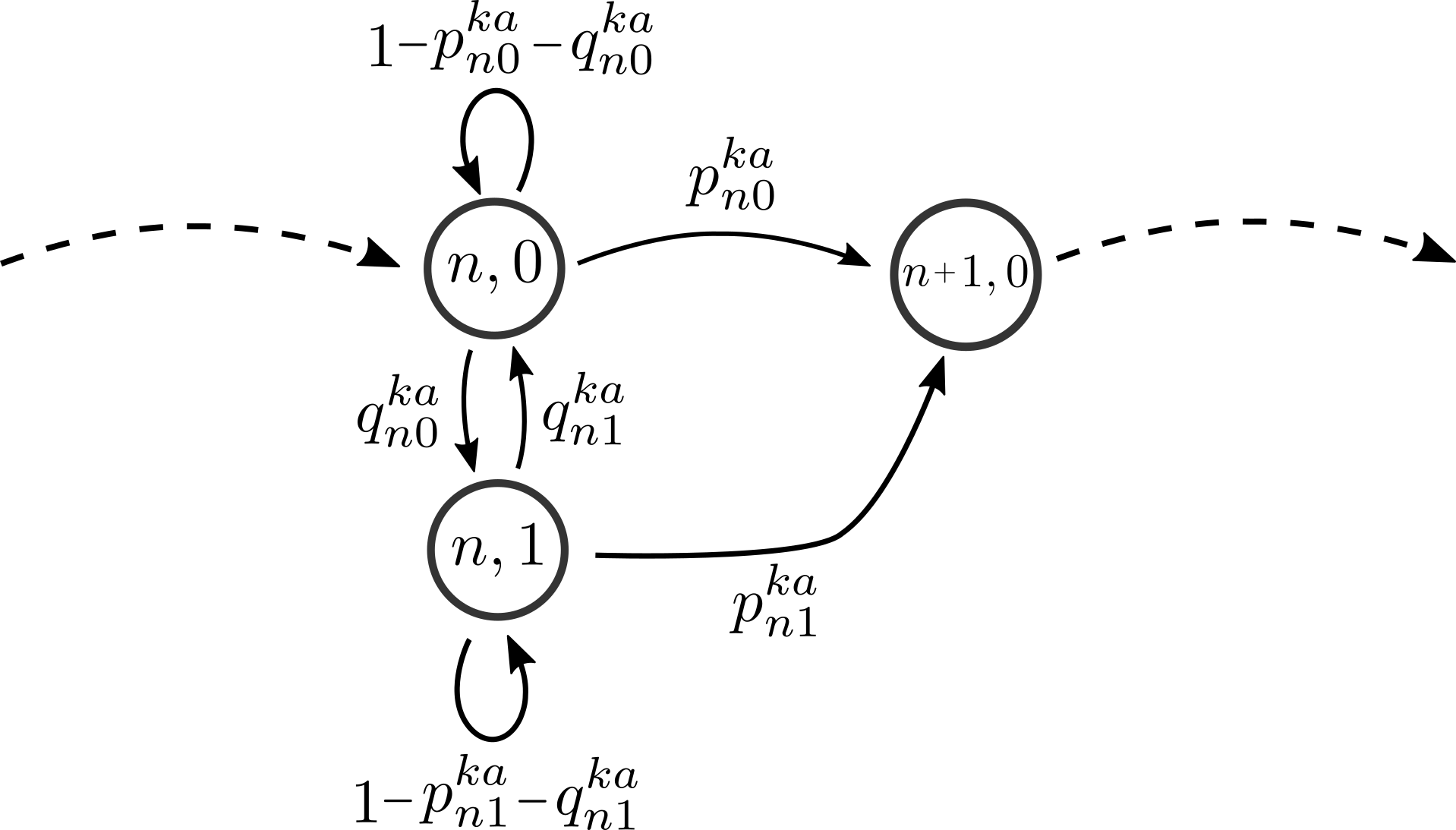}
		\caption{State-transition diagram for robot $k$ working on task $n$, where in $(n, 0)$ the robot is in the normal state $s=0$, and in $(n,1)$ the robot is in the fault state $s=1$. Transitions can occur between $(n,0), (n,1)$ and $(n+1,0)$, and the probabilities change with operating mode $a$.}
		\label{fig:state_transition}
	\end{figure}
	
	There is a per-step cost $C^k \colon \mathcal{X}^k \times \{0, 1\} \to \real_{\ge 0}$, where $C^k( (n^k, s^k), a^k)$ denotes the cost of operating robot $k \in \mathcal{K}$ in mode $a^k$ when the robot is in state $(n^k, s^k)$. Note that the per-step cost is zero in the terminal state, i.e, $C^k( (G,0), a) = 0$. 
	
	\subsection{Model of the decision support system (DSS)}
	There is a decision support system that helps to allocate operators to the robots. At each time the decision support system observes the operating state $\boldsymbol X_t \coloneqq (X^1_k, \dots, X^K_t)$ of all robots and picks at most $M$ robots to teleoperate. We capture this by the allocation $\boldsymbol A_t=(A^1_t,\ldots,A^K_t)\in \mathcal{A}$, where 
	\begin{equation}
		\mathcal{A} = \biggl\{ \boldsymbol{a}\coloneqq(a^1, \ldots, a^K) \in \{0, 1\}^{K} : \sum_{k=1}^{K} a^k \leq M \biggr\}. \label{eqn:action_space}
	\end{equation}
	The allocation is selected according to a time-homogeneous Markov policy $\pi : \pmb{\mathcal{X}} \to \mathcal{A}$. The expected total cost incurred by any policy $\pi$ is given by
	\begin{equation}
		J(\pi) = \mathbb{E}^\pi \biggl[\sum_{t=0}^{\infty} \gamma^t \sum_{k=1}^K C^k(X^k_{t}, A^k_{t}) \biggm| \boldsymbol X_0 = \boldsymbol x_0 \biggr],
		\label{eq:policy-cost}
	\end{equation}
	where $\gamma \in (0,1)$ is the discount factor and $\boldsymbol x_0=(x^1_0,\ldots,x^K_0)$ is the initial state with $x^k_0=(1,0)$ for every $k\in{\cal K}$.
	\subsection{Problem objective}
	We impose the following assumptions on the model:
	\begin{enumerate}
		\item [\textbf{(A1)}] Given an allocation $\boldsymbol{a} = (a^1, \dots, a^K)$ by the DSS, the operating states of the robots evolve independently of each other.
		\item [\textbf{(A2)}] For every robot $k \in \mathcal{K}$, the probability of getting out the faulty internal state when teleoperated is strictly greater than~$0$, i.e., $p^{k1}_{n1} + q^{k1}_{n1} > 0$.
		\item [\textbf{(A3)}] Under autonomous operation, a robot stays in the fault state, i.e., $p^{k0}_{n1} = q^{k0}_{n1} = 0$. 
	\end{enumerate}
	
	The design objective is to solve the following optimization problem:
	\begin{problem}\label{prob:main}
		Given the set $\mathcal{K}$ of robots, the system dynamics and the per-step costs, the number $M$ of human operators, and the discount factor $\gamma \in (0,1)$, choose a policy $\pi : \pmb{\mathcal{X}} \to \mathcal{A}$ to minimize the total discounted cost $J(\pi)$ given by~\eqref{eq:policy-cost}.
	\end{problem}
	Optimal solution for Problem~\ref{prob:main} can be found by modelling it as a Markov decision process and solving using dynamic programming~\cite{puterman2014markov}. However, the sizes of state and action spaces of the resulting model grows exponential with the number of robots and operators. Thus, solving Problem~\ref{prob:main} using dynamic programming becomes intractable for larger systems. To address this, we model Problem~\ref{prob:main} as a restless multi-armed bandit (RMAB) problem  and use the notion of indexability to find an efficient and scalable policy. We start by providing an overview of RMAB in the next section.
	
	\section{Overview of Restless Multi-armed bandits} \label{sec:restless}
	In this section we provide an overview of Restless Multi-Armed Bandits (RMAB), 
	indexability and the Whittle index policy.
	\subsection{Restless Bandit Process} \label{subsec:RB}
	A restless bandit (RB) process is a controlled Markov process~$(\mathcal{\tilde Z}, \{0, 1\}, \tilde T, \tilde C, \tilde z_0)$ where $ \mathcal{\tilde Z}$ is the state space, $\{0, 1\}$ is the action space, $\tilde T \colon \mathcal{\tilde Z} \times \mathcal{\tilde Z} \times \{0, 1\} \rightarrow \real_{[0,1]}$ is the transition probability function, $\tilde C: \mathcal{\tilde Z} \times \{0, 1\} \to \mathbb{R}$ is the per-step cost function, and $\tilde z_0$ is the initial state. 
	By convention, action~$0$ is called the \textit{passive} action and action~$1$ is called the \textit{active} action. 
	%
	\subsection{Restless Multi-armed Bandit Problem} \label{subsec:mRB}
	A Restless Multi-armed Bandit (RMAB) is a collection of $K$ independently evolving RBs~$(\mathcal{\tilde Z}^k, \{0, 1\}, \allowbreak \tilde T^k, \tilde C^k, \tilde z^k_0)$, $k \in {\cal K} \coloneqq \{1,\ldots,K\}$. Each process is conventionally called an \emph{arm}. A decision-maker selects at most $M$ arms ($M < K$) at each time instance. Let $\tilde Z^k_t$ and $\tilde A^k_t$ denote the state of arm~$k$ and the action chosen for arm~$k$ at time $t$. Let $\{\boldsymbol{\tilde{Z}}_t\}_{t \ge 0}$ and $\{\boldsymbol{\tilde{A}}_t\}_{t \ge 0}$ where 
	\begin{align*}
		\boldsymbol{\tilde{Z}}_t \coloneqq (\tilde Z^1_t,\ldots,\tilde Z^K_t) \;\text{ and }\;
		\boldsymbol{\tilde{A}}_t \coloneqq (\tilde A^1_t,\ldots,\tilde{A}^K_t),
	\end{align*}
	denote the states and actions of all arms.
	As the dynamics of each arm are independent, we have
	\[
	\tilde{T}( \boldsymbol{\tilde{Z}}_{t+1}| \boldsymbol{\tilde{Z}}_t,\boldsymbol{\tilde{A}}_t)
	=
	\prod_{k \in {\cal K}} \tilde T^k(\tilde{Z}^k_{t+1}| \tilde{Z}^k_t,\tilde{A}^k_t).
	\]
	The instantaneous cost of the system is the sum of costs incurred by each arm. 
	The performance of any time homogeneous Markov policy $\boldsymbol{\tilde{\pi}}:\prod_{k = 1}^{K}{\cal Z}^k\to\{\boldsymbol{a} \in\{0,1\}^K:~||\boldsymbol{a}||_1\leq M\}$ is measured by 
	\begin{equation}
		\tilde J(\boldsymbol{\tilde{\pi}}) = \mathbb{E}\biggl[ \sum_{t = 0}^{\infty} \gamma^t \sum_{k=1}^K \tilde C^k(\tilde Z^k_t, \boldsymbol{\tilde{\pi}}(\tilde Z^k_t)) \bigg| \tilde z^1_0, \ldots, \tilde z^K_0 \biggr], \label{eqn:obj_func-RB}
	\end{equation}
	where $\gamma \in (0,1)$ denotes the discount factor. Finally, the RMAB optimization problem is as follows:
	\begin{problem} \label{prob:infinite-RB}
		Given a discount factor $\gamma \in (0,1)$, a collection of arms~$\{(\mathcal{\tilde Z}^k, \{0, 1\}, \tilde T^k, \tilde C^k, \tilde z^k_0)\}_{k \in {\cal K}}$, and the number~$M$ of arms to be chosen at each time, choose a policy $\tilde {\boldsymbol \pi}:~\prod_{k = 1}^{K}{\cal Z}^k \to\{\boldsymbol{a} \in\{0,1\}^K:~||\boldsymbol{a}||_1\leq M\}$ that minimizes $\tilde{J}(\boldsymbol{\tilde{\pi}})$.
	\end{problem}
	As discussed earlier, to tackle the scalability issues of dynamic programming-based solutions, the Whittle index policy is one of the commonly used heuristic to solve a RMAB problem~\cite{whittleIndex}. 
	This policy is computationally efficient and it readily generalizes to the setting where $K$ or $M$ changes over time. Next, we present the required definitions.
	
	\subsection{Indexability and the Whittle index policy}
	In this section, we restrict our discussion to a single arm and therefore omit the superscript $k$ for the ease of notation. Consider an arm~$(\mathcal{\tilde Z}, \{0, 1\}, \tilde T, \tilde C_\lambda, {\tilde z}_0)$ where, for some penalty $\lambda\in\real$, modify the per-step cost as 
	\begin{equation} \label{eqn:modif_cost}
		\tilde C_\lambda(z, a) \coloneqq \tilde C(z, a) + \lambda a, \;
		~\forall\; z \in \mathcal{\tilde Z}, a \in \{0, 1\}.
	\end{equation}
	Then the performance of any given time-homogeneous Markov policy~$\tilde \pi: \mathcal{\tilde Z} \to \{0, 1\}$ is given by
	\begin{equation}
		\tilde{J}_{\lambda}(\tilde \pi) := \mathbb{E}\biggl[ \sum_{t = 0}^{\infty} \gamma^t \tilde C_\lambda(\tilde{Z}_t, \tilde \pi(\tilde{Z}_t)) \bigg| \tilde{Z}_0  \sim \tilde z_0 \biggr]. \label{eqn:obj_func-modif}
	\end{equation}
	Now consider the following auxiliary problem:
	\begin{problem}\label{prob:info-RB}
		Given an arm~$(\mathcal{\tilde Z}, \{0, 1\}, \tilde T, \tilde C, \tilde z_0)$, the discount factor~$\gamma \in (0,1)$ and the penalty $\lambda \in \mathbb{R}$, choose a Markov policy $\tilde \pi: \mathcal{\tilde Z} \to \{0, 1\}$ to minimize $\tilde{J}^{(\tilde \pi)}_{\lambda}(\tilde z_0)$ given by \eqref{eqn:obj_func-modif}.
	\end{problem}
	
	Problem~\ref{prob:info-RB} is a Markov decision process. Let us denote the optimal policy of Problem~\ref{prob:info-RB} by $\tilde \pi^~_\lambda$. It is assumed that the optimal policy picks passive action at any state where both the active and passive actions result in same expected cost. Next, define passive sets and indexability. 
	%
	\begin{definition}[Passive set]
		Given $\lambda\in\real$, the passive set~$\tilde {\cal P}(\lambda)$ is the set of states where passive action is prescribed by $\tilde \pi^~_\lambda$, i.e.,
		\begin{equation*}
			\tilde {\cal P}(\lambda) := \left\{ z \in {\cal Z}: \tilde \pi^~_\lambda(z) = 0 \right\}.
		\end{equation*}
		\label{def:passive_set}
	\end{definition}
	\begin{definition}[Indexability]
		An arm is indexable if $\tilde {\cal P}(\lambda)$ is non-decreasing in $\lambda$, i.e., for any $\lambda_1, \lambda_2 \in \mathbb{R}$,
		\begin{equation*}
			\lambda_1 \leq \lambda_2 \implies \tilde {\cal P}(\lambda_1) \subseteq \tilde {\cal P}(\lambda_2).
		\end{equation*}
		A RMAB problem is indexable if all $n$ arms are indexable.
		\label{def:indexability}
	\end{definition}
	\begin{definition}[Whittle index] 
		For an indexable arm, the Whittle index of the state~$z$ of an arm is the smallest value of $\lambda$ for which state~$z$ is part of $\tilde {\cal P}(\lambda)$, i.e.,
		\begin{equation}
			\tilde w(z) = \inf \left\{ \lambda \in \mathbb{R}: z \in \tilde {\cal P}(\lambda)\right\}.
			\label{eq:whittle_index_def}
		\end{equation}
		Equivalently, the Whittle index $\tilde w(z)$ is the smallest value of $\lambda$ for which $\tilde \pi^~_\lambda$ is indifferent between the active action and passive action when the arm is in state~$z$. 
		\label{def:whittle_index}
	\end{definition}
	%
	The Whittle index policy is as follows: \textit{At each time, compute the Whittle indices of the current state of all arms and select the arms in states with $M$ highest Whittle indices (provided they are positive).}

	\section{Indexability of the assistance problem}
	%
	
	Problem~\ref{prob:main} can be formulated as an instance of RMAB, where each robot corresponds to an arm. Under such a formulation, the state $\tilde{Z}^k_t$ of arm $k$ corresponds to operating state $x^k_t=(n^k_t,s^k_t)$ of robot $k$. The transition function $\tilde{T}^k$ corresponds to the robot state evolution shown in Fig.~\ref{fig:state_transition} and the cost function $\tilde{C}^k$ corresponds to the associated per-step cost $C^k$. In addition, allocating an operator to robot corresponds to choosing the active action for that arm while autonomous operation corresponds to choosing the passive action. This motivates using the Whittle index policy to solve Problem~\ref{prob:main}. 
	However, before we can implement this approach, we must check for indexability of the problem. \blu{As discussed earlier, there is no universal framework to verify indexability of a problem. Moreover, the optimal policy for the given problem does not show any threshold-based behaviour. Therefore, we determine sufficient conditions for indexability from first principles by using properties of the value function of each individual arm.}
	
	Since indexability has to be checked for each arm separately, for this analysis, we drop the superscript~$k$ from all variables.
	
	Let $V_\lambda : {\cal X} \to \mathbb{R}$ be the unique fixed point of the following equation
	\begin{equation*}
		V_\lambda(x) = \min_{a \in \{0, 1\}}  Q_\lambda(x, a),
	\end{equation*} 
	where 
	\begin{align}
		Q_\lambda(x, a) & = C(x, a) + \lambda a + \gamma \sum_{x' \in {\cal X}} {T}(x'|x,a)  V_\lambda(x'),
		\label{eq:Q_value_expanded}
	\end{align}
	represents the \emph{Q}-value of taking action $a$ in state $x$. Here the transition function $T(x'|x,a)$ denotes the probability of transition from state $x$ to state $x'$ under action $a$ and is represented by  Fig.~\ref{fig:state_transition}.
	Let $\pi_\lambda: {\cal X} \to \{0, 1\}$ be the corresponding optimal policy
	\begin{equation*}
		\pi_\lambda(x) = \arg \min_{a \in \{0, 1\}}  Q_\lambda(x, a).
	\end{equation*} 
	To ensure uniqueness of the $\arg\min$, we follow the convention that when $Q_\lambda(x, 0) = Q_\lambda(x, 1)$, the passive action $a=0$ is chosen.
	Let ${\cal P}(\lambda)$ be the passive set given penalty~$\lambda$ and $w(x)$ be the Whittle index of state~$x$ for the problem of operator allocation in a single-robot system.
	Furthermore, define the \textit{Benefit function} as
	\begin{equation}
		B_\lambda(x) = Q_\lambda(x, 1) - Q_\lambda(x, 0).
		\label{eq:B_lambda}
	\end{equation}
	Then, a sufficient condition for indexability is as follows:
	\begin{lemma} \label{lemma:Bindexable}
		A sufficient condition for Problem~\ref{prob:main} to be indexable is that the benefit function $B_\lambda(x)$ for each robot is monotonically increasing in $\lambda$ for all states $x\in \mathcal{X}$.
		\label{th:non-dec_B}
	\end{lemma}
	
	\begin{proof}
		The result follows from the observation that using~\eqref{eq:B_lambda} and Def.~\ref{def:passive_set}, we can re-write the passive set as
		\begin{equation}
			\mathcal{P}(\lambda) = \left\{ x \in \mathcal{X}: B_\lambda(x) \geq 0 \right\}.
			\label{eq:redef_passive}
		\end{equation}
		
		%
		Thus, monotonicity of the benefit function $B_{\lambda}(x)$ implies that the condition for indexability given in Def.~\ref{def:indexability} is satisfied.
		
	\end{proof}
	
	%
	\blu{We verify the monotonicity of $B_\lambda(x)$ by bounding the value function and establish the following:}
	%
	%
	%
	
	
	%
	\begin{theorem}\label{th:single_robot_indexability}
		Let $r^a_{ns} \triangleq 1-p^a_{ns}-q^a_{ns}$ denote the probability of repeating a task $n$ under mode $a$ with internal state $s$.
		Define $\alpha_1(n)$ and $\beta_0(n)$ as follows:
		%
		\begin{multline*}
			\alpha_1(n) = 1 + \frac{\gamma q^1_{n0}}{1 - \gamma \,r^1_{n1}} \\
			+ \frac{\gamma q^0_{n0} \left(\gamma \,r^1_{n0} + \cfrac{\gamma^2 q^1_{n0}q^1_{n1}}{1 - \gamma \,r^1_{n1}} -1\right)}
			{1 - \gamma \,r^1_{n1} -\gamma \,r^0_{n0} +\gamma^2 r^1_{n1} r^0_{n0} -\gamma^2 q^0_{n0} q^1_{n1}},
		\end{multline*}
		and 
		\begin{equation*}
			\beta_0(n) = \frac{\gamma (p^1_{n0} - p^0_{n0}) + \gamma^2 {\left(p^0_{n0} r^1_{n0} - p^1_{n0} r^0_{n0} \right)}}{1 - \gamma \,r_{n0}^0}.
		\end{equation*}
		Then, the single-robot problem is indexable if for all $n \in \{1,2,\ldots,N\}$:
		\begin{align}
			\alpha_1(n) &\geq 0 \quad \text{and} \quad \frac{\beta_0(n)}{1-\gamma} \geq -1.
			\label{eq:cond_a1_geq0}
		\end{align}
	\end{theorem}
	\begin{proof}
		See Appendix.
	\end{proof}
	The multi-robot problem is indexable if the conditions given in Theorem~\ref{th:single_robot_indexability} hold true for all robots.
	In the next section, we present specific instances of the general model described in Section~\ref{sec:general_formulation} which are indexable and discuss their relevance in practical assistance problems for (semi)autonomous delivery robots.
	%
	\section{Special Cases: Robot transitions in the city}
	\label{sec:example_problems}
	This section presents two specific classes of transition functions which represent two types of faults commonly occurring in systems with remote navigating robots. 
	\subsection{Transition Type-1 : Faults with continuation}
	\label{sec:problem-2}
	Consider the following transition behaviour along a robot's waypoints. At each time step, the robot moves to its next waypoint with a probability representing, for example, the crowd in the area. There is also a probability of getting into a fault state such as encountering an unidentifiable obstacle. 
	A human operator can teleoperate the robot to its next waypoint both from a normal or fault state. \blu{Such transitions represent faults where the robot is functioning properly but is unsure about how to proceed due to uncertainty in its surroundings. Thus the probability of success when being teleoperated is the same regardless of whether the robot is in its normal state or stopped in the fault state, i.e., $p^{1}_{n0}=p^{1}_{n1}$ and $q^{1}_{n0}=q^{1}_{n1}=0$.} The corresponding transition dynamics are shown in Figure~\ref{fig:state_transition_type-2}. Note that in this case $r^{1}_{n0}=r^{1}_{n1}=1 - p^{1}_{n0}$.
	
	\begin{figure}[ht]
		\centering
		\includegraphics[width=0.75\columnwidth]{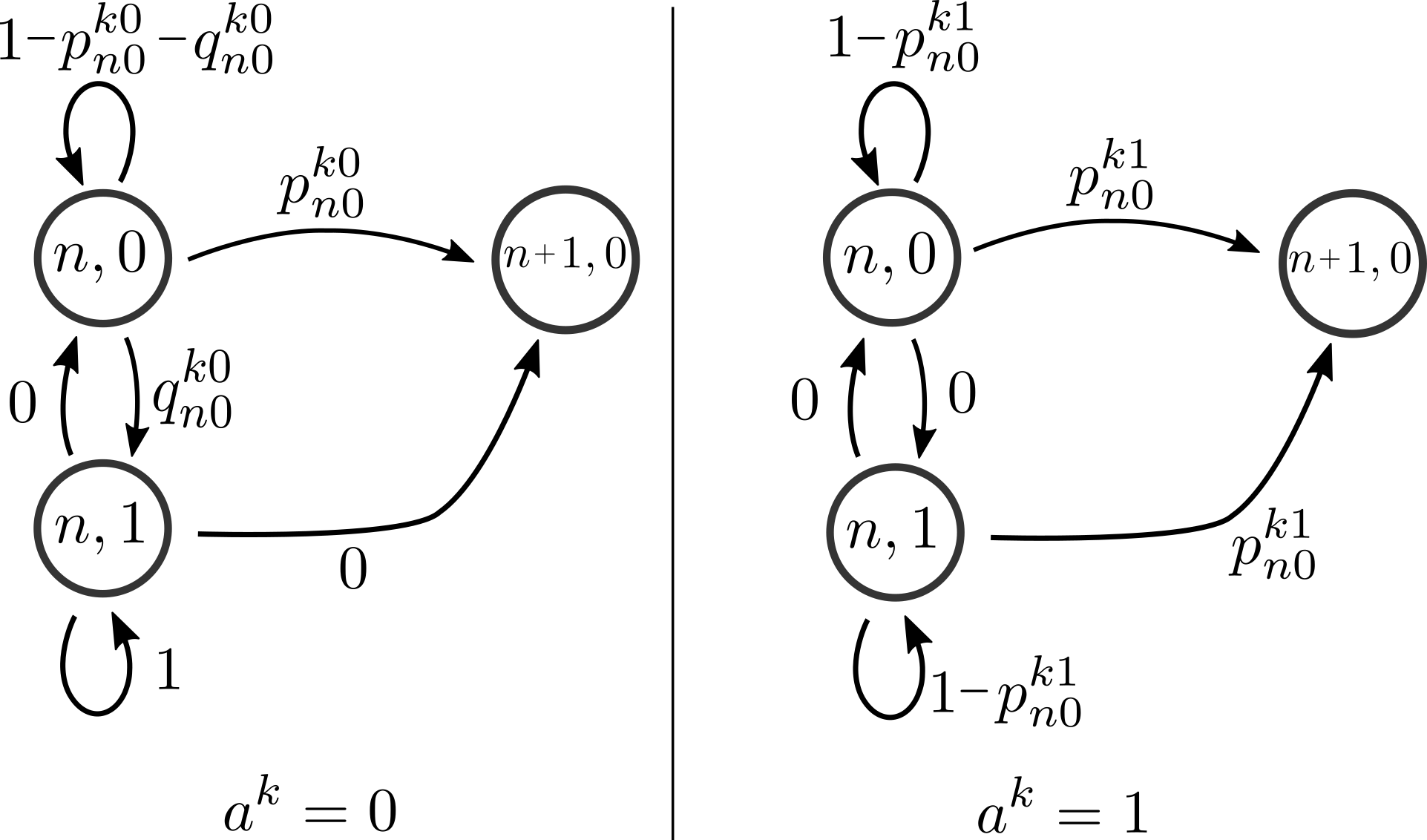}
		\caption{State-transition probabilities under the autonomous operation and teleoperation for type-1 transitions. 
		}
		\label{fig:state_transition_type-2}
	\end{figure}
	
	In this case, the coefficients $\alpha_1(n)$ and $\beta_0(n)$ can be simplified to the following expressions:
	\begin{align}
		\alpha_1(n) &= 1 - \frac{\gamma q^0_{n0}}{1-\gamma r^0_{n0}}, \nonumber \\[\lessverticaldistance]
		\beta_0(n) &=  \frac{\gamma (1-\gamma)(r^0_{n0}- r^1_{n0}) + \gamma q^0_{n0}(1-\gamma\,r^1_{n0})}{1-\gamma \,r^0_{n0}}.
		\label{eq:a1a2_prob2}
	\end{align}
	Note that 
	\begin{align*}
		\alpha_1(n) 
		&= \frac{1-\gamma + \gamma p^0_{n0}}{1 - \gamma r^0_{n0}}
		\ge 0,
		\\
		\frac{\beta_0(n)}{1 - \gamma} + 1 &\ge
		\frac{\gamma(r^0_{n0} - r^1_{n1})}{1 - \gamma r^0_{n0}} + 1
		= \frac{1 - \gamma r^1_{n1}}{1 - \gamma r^0_{n0}} \ge 0.
	\end{align*}
	Thus, $\alpha_1(n)$ and $\beta_0(n)$ satisfy the sufficient condition of Theorem~\ref{th:single_robot_indexability} for all allowed values of transition probabilities and the discount factor $\gamma$. Therefore, any robot following the Type-1 transitions is indexable. 
	
	
	
	\subsection{Transition Type-2 : Faults with reset}
	\label{sec:problem-3}
	Consider another type of transition 
	where the robot can get into a fault state and needs error fixing while staying at its next waypoint. This includes scenarios such as losing localization or getting stuck in a minor obstacle. The human operator can try to assist the robot out of that situation by fixing the fault, resetting it back to its current waypoint (assuming the system is equipped with means to do so). Such transitions will mean that the probabilities $q^{1}_{n0}=p^{1}_{n1}=0$ and the corresponding transition dynamics are shown in Fig.\ref{fig:state_transition_type-3}:
	%
	\begin{figure}[ht]
		\centering
		\includegraphics[width=0.75\columnwidth]{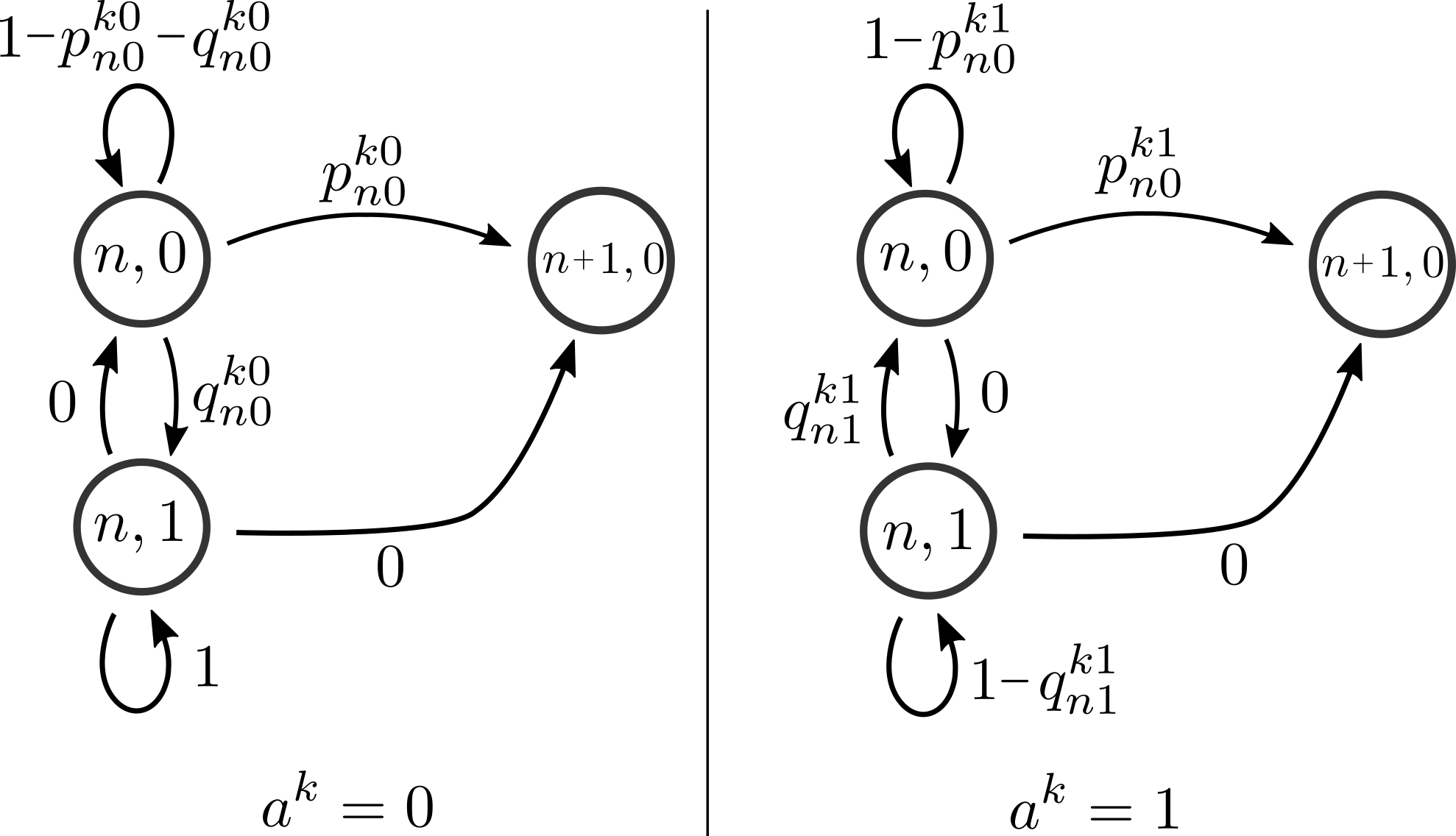}
		\caption{State-transition probabilities under the autonomous and assist/teleoperate actions for Type-2 transition dynamics. 
		}
		\label{fig:state_transition_type-3}
	\end{figure}
	
	\blu{Substituting the values of transition probabilities from Fig.~\ref{fig:state_transition_type-3} to the expressions of $\alpha_1(n)$ and $\beta_0(n)$, the coefficients can be simplified to the following}:
	\begin{align}
		\alpha_1(n) &= 1 - \frac{\gamma q_{n0}^0 (1 - \gamma r^1_{n0})}{(1-\gamma r^0_{n0})(1-\gamma (1-q^1_{n1})) - \gamma^2 q^0_{n0} q^1_{n1}}, \nonumber \\
		\beta_0(n) &=  \frac{\gamma (1-\gamma)(r^0_{n0}- r^1_{n0}) + \gamma q^0_{n0}(1-\gamma\,r^1_{n0})}{1-\gamma \,r^0_{n0}}.
		\label{eq:a1a2_prob3}
	\end{align}
	Note that $\beta_0(n)$ here is the same as~\eqref{eq:a1a2_prob2} and, therefore, satisfies~\eqref{eq:cond_a1_geq0}. For $\alpha_1(n)$ to satisfy~\eqref{eq:cond_a1_geq0}, \blu{the condition $\alpha_1(n) \geq 0$ results in the following condition on $q^1_{n1}$:}
	\begin{align}
		%
		q^1_{n1} &\geq 1 - \frac{1}{\gamma} + \frac{\gamma q^0_{n0}p^1_{n0}}{1 - \gamma r^0_{n0}- \gamma q^0_{n0}}.
		\label{eq:cond_a1_prob3}
	\end{align}
	%
	As $q^1_{n1} \leq 1$,~\eqref{eq:cond_a1_prob3} also yields the following condition on $q^0_{n0}$:
	\begin{equation}
		q^0_{n0} \leq \frac{1-\gamma r^0_{n0}}{\gamma(1 +\gamma p^1_{n0})}.
		\label{eq:cond_a1_prob3_p1i}
	\end{equation}
	Therefore, any robot state following the Type-2 transitions will satisfy the condition of indexability in Theorem~\ref{th:non-dec_B} if~\eqref{eq:cond_a1_prob3} and \eqref{eq:cond_a1_prob3_p1i} are satisfied, i.e., the probability $q^0_{n0}$ that the robot transitions from a normal state to fault state during autonomous operation is not too high and the probability $q^1_{n1}$ that the operator brings the robot from a fault state to a normal state is not too small.
	
	As an example, consider a robot following Type-2 transitions with $p^0_{n0} = q^0_{n0} = p^1_{n0} = 0.3$ and $\gamma=0.95$. In this setting, any $q^0_{n0} \in [0,1]$ satisfies \eqref{eq:cond_a1_prob3_p1i} and any $q^1_{n1} \in [0.1462,1]$ satisfies~\eqref{eq:cond_a1_prob3}. Thus, the model is indexable if there is at least a $14.62\%$ chance that teleoperation successfully resets the robot from the fault state to a normal state.

	\section{Computation of Whittle Index}
	\label{sec:index_policy}
	As discussed in Section \ref{sec:background}, once the indexability of the problem instance has been verified, we can compute Whittle indices for all robots and determine the Whittle index policy.
	
	
	General approaches of computing Whittle indices are either based on adaptive greedy algorithm \cite{nino2007dynamic, akbarzadeh2021conditions} or binary search \cite{akbarzadeh2019dynamic}. In this section, we briefly provide details on adaptive greedy algorithm and describe how the Whittle index policy works. Readers are encouraged to refer to \cite{akbarzadeh2021conditions} for a detailed explanation and validation of the algorithm. The algorithm is presented in Alg.~\ref{alg:Whittle_index_adaptive} for computing Whittle indices for a single robot.

	\begin{algorithm}[ht]
		\begin{algorithmic}[1]
			\STATE \textbf{Input:} Robot~$(\mathcal{X},\mathcal{A},T,C,\gamma, x_0)$.
			\STATE Initialize ${\cal P} = \emptyset$.
			\WHILE{${\cal P} \neq {\cal X}$}
			\STATE Compute $\mu^*_{y}$, $\forall y \in {\cal X} \backslash {\cal P}$ using Eq.~\eqref{eq:mu_y}.
			\STATE $\lambda^* \leftarrow \min_{y \in {\cal X} \backslash {\cal P}} \;\mu^*_{y}$ \label{line:lambda^*}
			\STATE $\mathcal{Y} \leftarrow \arg \min_{y \in {\cal X} \backslash {\cal P}} \;\mu^*_{y}$ \label{line:Z}
			\STATE $w(y) \leftarrow \lambda^*$, $\forall y \in \mathcal{Y}$
			\STATE ${\cal P} \leftarrow {\cal P} \cup \mathcal{Y}$ 
			\ENDWHILE
		\end{algorithmic} 
		\caption{Adaptive Greedy Algorithm for Whittle Index Computation}
		\label{alg:Whittle_index_adaptive}
	\end{algorithm}
	
	%
	The algorithm operates as follows: 
	For any subset ${\cal Z} \subseteq {\cal X}$, define the policy vector\footnote{\blu{In the following expressions $\pi$ is used as a vector of size $|\mathcal{X}|$, constructed as a mapping from each state to corresponding action $a\in\{0,1\}$.}} $\pi^{({\cal Z})}: {\cal X} \to \{0, 1\}$ as 
	\begin{equation*}
		\pi^{(\cal Z)}(x) = 
		\begin{cases}
			0, ~ \text{ if } x \in {\cal Z} \\
			1, ~ \text{ if } x \in {\cal X}\backslash{\cal Z}.
		\end{cases}
	\end{equation*}
	Also define, $C_{\pi} = \pig[C(x, \pi(x))\pig]_{x \in {\cal X}}$, the cost vector for all states under a policy $\pi$, and $T_\pi = \pig[T(x'|x,\pi(x))\pig]_{x,x' \in {\cal X}}$, the transition matrix under policy $\pi$.
	
	Then, in each iteration of the while loop, compute $\mu^*_y$ as follows:
	\begin{align}  
		\mu_{y}(x) & = -\frac{D_{\pi^{({\cal P})}}(x) - D_{\pi^{({\cal P} \cup \{y\})}}(x)}{N_{\pi^{({\cal P})}}(x) - N_{\pi^{({\cal P} \cup \{y\})}}(x)}, ~ \forall x \in {\cal X} \nonumber\\[\lessverticaldistance]
		\mu^*_{ y} & = \min_{x \in \Lambda_{ y}} \mu_{ y}(x),
		\label{eq:mu_y}
	\end{align}
	where
	\begin{align*} 
		D_{\pi}(x) &= \pig[(I - \gamma T_{\pi})^{-1} C_\pi\pig](x)\;\;, \\
		N_{\pi}(x) &= \pig[(I - \gamma T_{\pi})^{-1} \pi\pig](x) \;\;, \\
		\Lambda_{y} &= \{s \in {\cal X}: N_{\pi^{({\cal P})}}(x) - N_{\pi^{({\cal P} \cup \{y\})}}(x) \neq 0\}.
	\end{align*}
	The minimum value of $\mu^*_y$ calculated in line~\ref{line:lambda^*} in Alg.~\ref{alg:Whittle_index_adaptive} corresponds to the Whittle indices of the minimizing states (Line~\ref{line:Z}). These states are then taken out of consideration in the next iteration of the while loop by including them in the passive set $\mathcal{P}$.
	%
	When the Alg.~\ref{alg:Whittle_index_adaptive} exits the while loop, the Whittle indices for all states of that robot are calculated. This procedure is then repeated for all the robots in the system.
	
	Once the Whittle indices for all states of all robots are obtained, the Whittle index policy can be implemented as given in Alg.~\ref{alg:index_policy}. \blu{In line~\ref{alg:line:argtopM} of the algorithm, the function $\texttt{arg\_top\_M}(\{w^k(x^k)\})$ returns indices of top $M$ positive elements in a set, where ties are broken randomly.} As determined in \cite{akbarzadeh2021conditions}, the computational complexity of this method is ${\cal O}(K |{\cal X}|^3)$. In contrast, the computational complexity of finding the optimal policy for Problem~\ref{prob:main} is ${\cal O}({K \choose M} |{\cal X}|^{2K})$ using value iteration, where $|{\cal X}|$ is the size of state space of individual robot.
	%
	%
	\begin{algorithm}[h!]
		\begin{algorithmic}[1]
			\STATE \textbf{Input:} Set of Whittle indices $w^k(x^k)$ for all $k \in \{1,\ldots, K\}$ and $x^k \in \mathcal{S}^k$, No. of Operators $M$
			\STATE $\mathcal{M} \leftarrow \texttt{arg\_top\_M}(\{w^k(x^k)\})$ \label{alg:line:argtopM}
			
			\STATE $a^{k} \leftarrow 0$ for all $k\notin\mathcal{M}$
			\STATE $a^{k} \leftarrow 1$ for all $k\in\mathcal{M}$ \quad\quad \gray{// Allocate operators}
			\RETURN $(a^1,\ldots,a^K)$
		\end{algorithmic}
		\caption{Whittle Index Policy $\pi^{I}$} 
		\label{alg:index_policy}
	\end{algorithm}

	\section{Simulations and Results}
	\label{sec:simulation}
	
	
	
	
	In this section, we present performance results for a simulated multi-robot assistance problem under the following policies (described later): 1)~Optimal policy, 2)~Index policy, 3)~Benefit maximizing Policy, 4)~Myopic Policy, and 5)~Reactive Policy. 
	The problem and the solution frameworks for all policies were implemented using POMDPs.jl library in Julia \cite{egorov2017pomdps}.

	\subsection{Simulation Setup}
	For the simulations, a city map is generated as shown in Fig. \ref{fig:simulation_setup} where the map is randomly divided into different zones corresponding to one of the two transition types presented in Section~\ref{sec:example_problems}.
	\begin{figure}
		\centering
		\includegraphics[width=0.75\columnwidth]{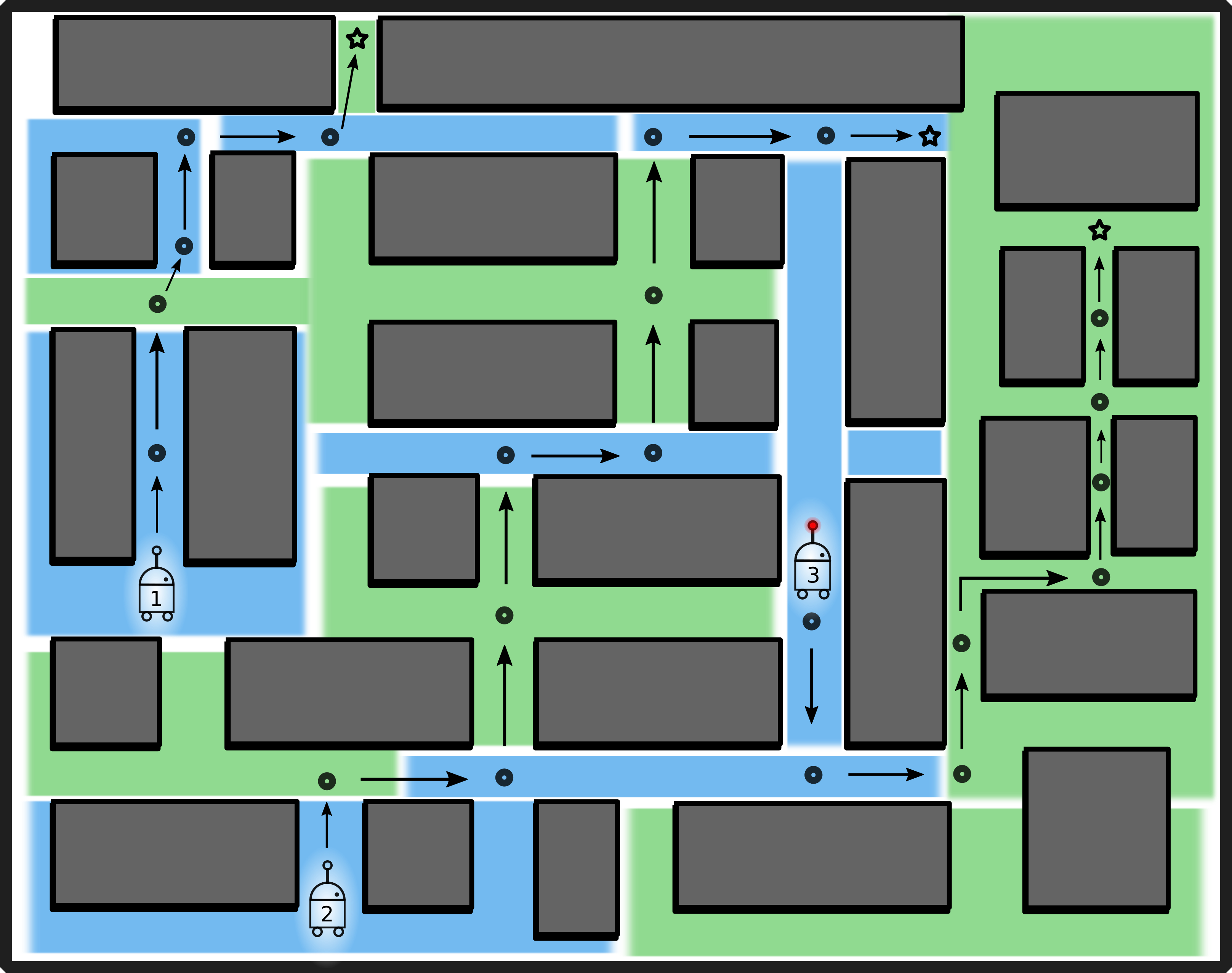}
		\caption{An instance of multi-robot assistance problem for robots navigating in a city block-like environment. Transition zones are marked by different color shadings. The green and blue areas represent type-1 and type-2 transitions respectively, as described in Section~\ref{sec:example_problems}. Three robots are shown navigating to their corresponding goal locations, via a sequence of waypoints shown as black circles. These waypoints may lie in different transition zones resulting in varied performance for the robots.}
		\label{fig:simulation_setup}
	\end{figure}
	
	The exact values of transition probabilities at different locations in the map are sampled randomly from a uniform distribution, according to Table~\ref{table:probs}. 
	The bounds on transition probabilities $\overline{q}^{k1}_{n1}$ and $\overline{q}^{k0}_{n0}$ for transition type-2 are determined by~\eqref{eq:cond_a1_prob3} and \eqref{eq:cond_a1_prob3_p1i} respectively.
	
	{\renewcommand{\arraystretch}{1.2}
		\begin{table}[ht]
			\centering
			\caption{Probabilities values used for simulations for the two types of transitions.}
			\label{table:probs}
			\begin{tabular}[pos=c]{lll} 
				\toprule
				Probability & Type-1 & Type-2\\
				\midrule
				$r^{k0}_{n0}$ & $[0.2,0.5]$ & $[0.2,0.5]$\\[0.1ex]
				$q^{k0}_{n0}$ & $[0.2,0.5]$ & $[0.1, \min\{\overline{q}^{k0}_{n0}, 1-r^{k0}_{n0}\}]$\\[0.1ex] 
				$r^{k1}_{n0}$ & $[0.1,0.4]$ & $[0.1,0.4]$ \\[0.1ex]
				$q^{k1}_{n0}$ & $0.0$ & $0.0$  \\[0.1ex]
				$r^{k1}_{n1}$ & $r^{k1}_{n0}$ & $1-q^{k1}_{n1}$\\[0.1ex] 
				$q^{k1}_{n1}$ & $0.0$ & $[\max\{\overline{q}^{k1}_{n1}, 0.1\}, 0.9]$ \\[-0.3ex]
				\bottomrule
			\end{tabular}
			\vspace{-2ex}
		\end{table}
	}
	For the teleoperation problem, we use the following cost structure:
	\begin{equation}
		C^k\left((n,s),a\right) = \begin{cases}
			0 & \text{if } n = G \\
			\rho^k_n &  \text{if } a=0, s = 0 \\
			\phi^k_n &  \text{if } a=0, s = 1 \\
			\rho^k_n +\rho_T^k &  \text{if } a=1, s = 0 \\
			\phi^k_n +\rho_T &  \text{if } a=1, s = 1,
		\end{cases}
		\label{eq:costFn}
	\end{equation}
	with $\rho^k_n, \phi^k_n, \rho_T^k \in \real_{\geq0}$ for any $n \in \{1, \ldots, N^k\}$.
	This cost function captures the time that a robot takes to reach its goal, i.e., zero cost on reaching goal state, non-negative costs for the intermediate states, and a fixed additional cost of $\rho^k_{T}$ while being assisted. 
	
	Values of the different costs and the discount factor used are sampled from ranges specified in Table~\ref{table_values}.
	{\renewcommand{\arraystretch}{1.0}
		\begin{table}[ht]
			\centering
			\caption{Parameter values used in the simulated assistance task.}
			\label{table_values}
			\begin{tabular}[pos=c]{ccccc} 
				\toprule
				Parameter & $\rho^~_T$ & $\rho$ & $\rho_e$ & $\gamma$\\ [0.1ex] 
				\midrule
				Value & $0.75$ & $2.0$ & $4.0$ & $0.99$\\[0.1ex] 
				\bottomrule
			\end{tabular}
			\vspace{1ex}
			
			{\raggedright \blu{For simplicity, for a given parameter, same range is used for every state of all robots. Therefore, we remove the superscript $k$.} \par}
		\end{table}
	}
	
	
	Separate tests were performed to test the validity, performance and scalability of the Index policy. At the beginning of each simulation, a number of robots are placed on the map (ranging from $1$ to $25$) with randomly generated start and goal locations, and $7$ waypoints uniformly placed between the two. In practice, these waypoints are generated by a separate robot path planner for each individual robot, and are considered as an input for the operator allocation problem. 
	
	\subsection{Baseline Policies}
	We consider the following baseline policies to assess the performance of the Index policy.
	\subsubsection*{Optimal policy} 
	The Optimal policy $\pi^*: \pmb{\mathcal{X}} \rightarrow \mathcal{A}$, as defined by~\eqref{eq:policy-cost}, is found by encoding the complete problem with all robots as an MDP and solving it using the \emph{Sparse Value Iteration Solver} from the POMDP.jl library. 
	
	
	\subsubsection*{Reactive policy}
	The reactive policy simply allocates an operator to any robot stuck in a fault state. If there are more such robots than operators, a random subset of those robots is selected.
	
	\subsubsection*{Myopic policy} 
	Myopic/Greedy Policies are commonly used to obtain fast (but sub-optimal) solutions to intractable problems. 
	In this paper, we implement an $l$-step myopic policy presented in \cite{rosenfeld2017intelligent} for $l\in\{1,2\}$.
	%
	Define $V^{\mathbf{0}}(\boldsymbol{x}_{t+1})$ as the expected cost incurred by the system from current time step to infinity under passive actions. The $l$-step myopic policy $\pi^{G\mhyphen l}: \pmb{\mathcal{X}} \rightarrow \mathcal{A}$ is then defined as
	\begin{align}
		\pi^{G\mhyphen l}(\boldsymbol{x}_t) = \arg\min_{\boldsymbol{a}\in \mathcal{A}} \; g(\boldsymbol{x}_t, \boldsymbol{a}, l), 
		\label{eq:myopic_policy}
	\end{align}
	where the $l$-step look-ahead cost $g(\boldsymbol{x}_t,\boldsymbol{a},l)$ is given by
	\begin{equation*}
		g(\boldsymbol{x}_t, \boldsymbol{a}, l) = \begin{cases}
			{C}\left(\boldsymbol{x}_t,\boldsymbol{a}\right)+ \\\sum\limits_{\boldsymbol{x}_{t+1}} \gamma\; T(\boldsymbol{x}_{t+1}|\boldsymbol{x}_{t},\boldsymbol{a})\, g(\boldsymbol{x}_t, \boldsymbol{a}, l-1),~ \text{if } l = 0,
			\\
			V^{\mathbf{0}}(\boldsymbol{x}_{t+1}),\qquad\qquad\qquad\qquad\qquad\text{otherwise,}
		\end{cases}
	\end{equation*}
	where
	\[C\left(\boldsymbol{x},\boldsymbol{a}\right) = \sum_{k=1}^K C^k\left(x^k,a^k\right),\]
	is the cost incurred in the current time step.
	%

	\subsubsection*{Benefit maximizing policy}
	The Benefit maximizing policy is inspired by the \emph{advantage function} used in reinforcement learning, see for example~\cite{schulman2015high}. This policy considers the \textit{benefit} or \textit{advantage} of taking the active action over the passive action for each robot and picks the robots with highest benefit at each time step, i.e.,
	\begin{equation}
		\pi^B(\boldsymbol{x}_t) = \arg\min_{\boldsymbol{a} \in \mathcal{A}}\; \sum_{k=1}^K a^k \; B_0(x^k_t),
	\end{equation}
	where $B_0(x)$ corresponds to $B_\lambda(s)$ defined in~\eqref{eq:B_lambda} with $\lambda=0$.
	
	
	\subsection{Comparison with the Optimal policy} First, the Index policy is compared against the Optimal policy to validate its applicability for our problem. 
	%
	Due to its poor scalability, the Optimal policy cannot be computed for larger problem instances, therefore this test is limited to smaller problem size (up to $4$ robots and $2$ operators). The relative performance (ratio of the cost incurred under Optimal policy to that under Index policy) is shown in Fig.~\ref{fig:validity}.
	For comparison, $100$ problem instances were tested under both policies and were simulated through Monte Carlo rollouts. 
	\blu{Each problem instance is run until all robots reach their respective goal locations. This is repeated $100,000$ times and average cost is recorded.}
	\begin{figure}[ht]
		\centering
		\includegraphics[width=0.9\columnwidth]{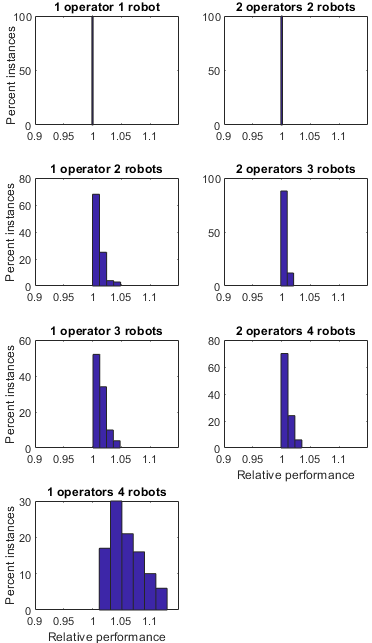}
		\caption{Relative performance of Index policy compared to the optimal policy. The plots show distribution of $100$ indexable problem instances based on their performance under the two policies. \blu{Relative performance is calculated as the ratio of the cost incurred under optimal policy to that under index policy.} 
		}
		\label{fig:validity}
	\end{figure}
	
	It is observed that the Index policy performs quite close to the optimal policy for all test cases. As the ratio of number of robots to number of operators increases, the Index policy starts to degrade in comparison to the optimal. However, the relative cost still remains within $1.13$ of the optimal. \blu{Also, note that the optimal policy minimizes the expectation of the cost incurred. The graph shows distribution of the relative performance, which is centered around 1. Therefore, we still see some test runs where the rollout resulted in a slightly lesser cost for the index policy. 
		(see Fig.~\ref{fig:validity}). However, as number of test runs are increased the average relative performance approaches $1$}.
	
	\subsection{Comparison with other baseline policies} 
	Next, we compare the performance \blu{(measured as average cost incurred per robot before reaching its goal)} of the Index policy with the three baseline policies on larger problem instances.
	For the comparison, a set of $100$ problem instances is created, each with a set of $7$ waypoints with randomly sampled transition probabilities according to Table~\ref{table:probs}. Each instance of the problem is then simulated separately under the different policies using Monte Carlo rollouts until all robots reach their goal states, repeated for 500 iterations. \blu{Each simulation iteration (rollout) is timed out at $10$ seconds for each policy. If an iteration takes longer than this time, the simulations are interrupted and the result for that test condition is not reported.}
	
	Figure~\ref{fig:comparison} shows performance comparison of the four policies. 
	The Index policy performs best out of the four policies, followed by the benefit maximizing policy ($\pi^B$) and the myopic policy ($\pi^G$). The Reactive policy performed the worst as expected. 
	As a side note, the average cost incurred per robot under any policy is strongly correlated with the ratio of number of robots to the number of available operators. This observation supports the intuition that as human operators are required to distribute their assistance among more robots, their effectiveness decreases.
	%
	%
	\begin{figure}
		\centering
		\includegraphics[width=0.9\columnwidth]{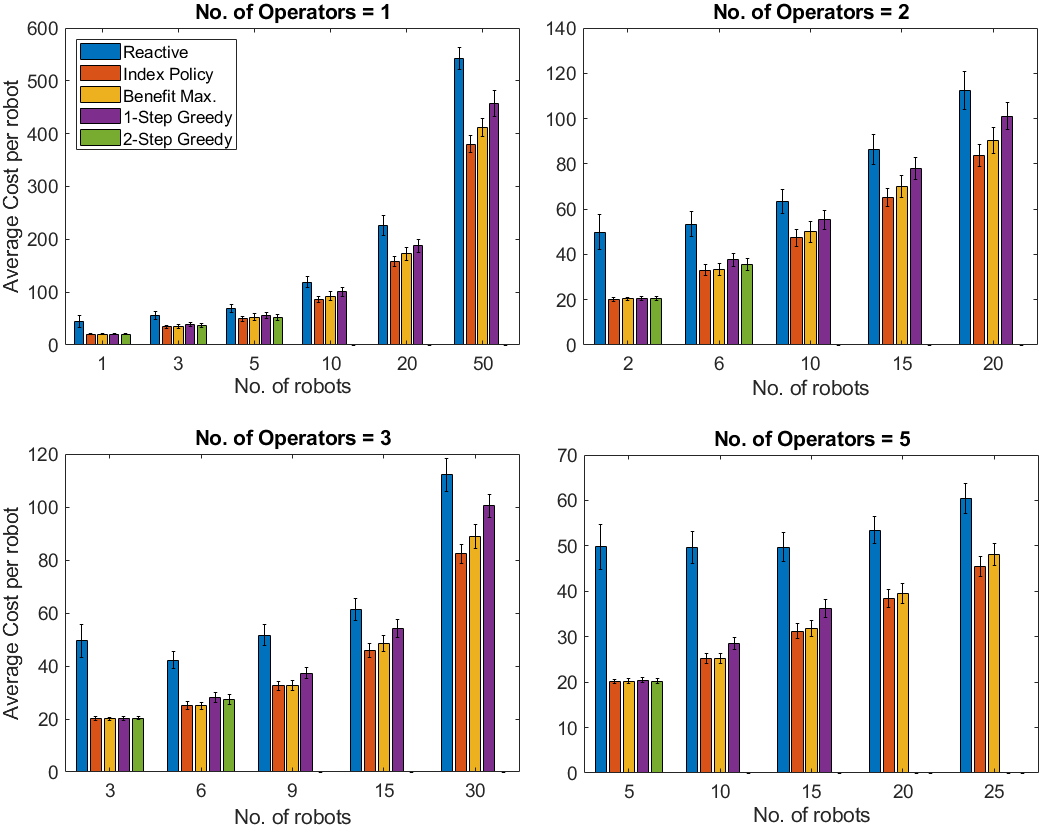}
		\caption{Performance comparison of the four policies for different number of operators available for allocation. Error bars in the plots show one standard-deviation above and below the average. \blu{Note that in larger problem instances, the simulations for the myopic policies timed out and could not be completed in the specified time limit (10 seconds per rollout)}.}
		\label{fig:comparison}
	\end{figure}
	\subsection{Scalability} \color{black}
	\blu{Table~\ref{table:times} shows the time that each policy takes to compute operator the allocation under different problem sizes. For these simulations, each robot is set to have $7$ waypoints.}
	As observed from Table~\ref{table:times}, the online computation times of the two Myopic policies scale exponentially with both the number of robots and the number of operators, with the time for $2$-step Myopic policy growing at a much higher rate.
		
		The computation times for the Index and Benefit maximizing policies scale linearly with the number of robots and are independent of the number of operators. 
		Also, note that the Whittle index computation for one robot is independent from the rest. Therefore robots can be added/removed without re-computation of already computed indices.
	Furthermore, if the number of operators change to $M > 1$, the policy simply allocates operators to the robots with the $M$ highest Whittle indices. As a result, the policy is efficiently scalable with the number of robots and operators.
	For reference, the simulations were run on a Desktop PC with a $4$ core, $4.20$ GHz processor and $32$ GB of RAM.
	
	
	{\renewcommand{\arraystretch}{1.0}
		\begin{table}[ht]
			\centering
			\caption{\blu{Computation times of different policies (seconds)}}
			\label{table:times}
			\begin{tabular}{p{0.15\linewidth} | p{0.16\linewidth} | p{0.16\linewidth} | p{0.16\linewidth} | p{0.16\linewidth}} 
				\toprule
				Operators/ \hspace*{\fill} Robots & Index Policy & $1$-step \mbox{Myopic} & $2$-step \mbox{Myopic} & Benefit Maximizing\\ \hline
				\rule{0pt}{2.5ex}$2/6$ & $1.2e^{-6}$ & $3.8e^{-3}$ & $2.4e^{-1}$ & $3.4e^{-6}$\\
				$3/6$ & $1.2e^{-6}$ & $6.2e^{-3}$ & $4.8e^{-1}$ & $3.4e^{-6}$\\
				$4/6$ & $1.2e^{-6}$ & $8.1e^{-3}$ & $6.6e^{-1}$ & $3.4e^{-6}$\\
				$1/9$ & $1.7e^{-6}$ & $9.6e^{-2}$ & $4.8e^{+0}$ & $4.9e^{-6}$\\
				$2/9$ & $1.7e^{-6}$ & $2.6e^{-1}$ & $2.2e^{+1}$ & $4.9e^{-6}$\\
				$3/9$ & $1.7e^{-6}$ & $6.4e^{-1}$ & $6.2e^{+1}$ & $4.9e^{-6}$\\[-0.5ex]
				\bottomrule
			\end{tabular}
		\end{table}
	}
	\vspace{-2ex}

	\section{Conclusions and Discussion}
	\label{sec:discussion}
	In this paper, we provide an analysis of operator allocation problem for a multi-robot assistance task and demonstrate the effectiveness of Restless Bandit framework to obtain a scalable policy. This policy is based on Whittle index heuristic and performs close to the optimal and significantly better than other efficient solution approaches. 
	We also provide an analysis of indexability of such problems and give a simplified condition to quickly verify if a problem instance is indexable. These results can also be used to specify required transition behavior in the form of bounds on transition probabilities.
	
	\blu{There are, however, a few limitations of the proposed approach. When a problem instance is not indexable, the Whittle indices are not defined and the methods of computing these indices may not give meaningful values. Therefore, the proposed approach is not applicable in such cases.} 
	\color{black}
	Also, note that the conditions for indexability identified in Theorem~\ref{th:single_robot_indexability} are sufficient but not necessary. \blu{However, the assistance problem presented here is mostly indexable. This was verified by randomly generating problem instances without the bounds given in Table~\ref{table:probs}
		and numerically verifying monotonicity of the passive set $\mathcal{P}(\lambda)$. Out of the $1000$ random instances, $999, 992$ and $940$ instances were found to be indexable for discount factors $\gamma = 0.9, 0.95, 0.99$ respectively.
		However, the conditions were satisfied for $700, 607$ and $452$ of these instances for $\gamma = 0.9, 0.95, 0.99$.
		This suggests that model is, in general, indexable and the Whittle index heuristic is applicable. It also suggests that the system may benefit from improved conditions for indexability.}
	
		\bibliographystyle{IEEEtran}

	\begin{IEEEbiography}{Abhinav Dahiya} is a PhD student in the Autonomous Systems Lab. in the department of Electrical and Computer Engineering, University of Waterloo, Canada. He received his Bachelor's degree in electrical engineering from Indian Institute of Technology Roorkee, India in 2016. His research interests include control and characterization of multi-agent systems involving human-robot interaction.
	\end{IEEEbiography}
	\vskip 0pt plus -1fil
	\begin{IEEEbiography}{Nima Akbarzadeh}
		(S’17) is a PhD student in the Electrical and Computer Engineering, McGill University, Canada. He received the B.Sc. degree in Electrical and Computer Engineering from Shiraz University, Iran, in 2014, the M.Sc. in Electrical and Electronics Engineering from Bilkent University, Turkey, in 2017. He is a recipient of 2020 FRQNT PhD Scholarship. His research interests include stochastic control, reinforcement learning and multi-armed bandits.
	\end{IEEEbiography}
	\vskip 0pt plus -1fil
	\begin{IEEEbiography}{Aditya Mahajan}
		(S’06-M’09-SM’14) is Associate Professor in the department of Electrical and Computer Engineering, McGill University, Montreal, Canada. He is Associate Editor for IEEE Transactions on Automatic Control and Springer Mathematics of Control, Signal, and Systems. He was an Associate Editor of the IEEE Control Systems Society Conference Editorial Board from 2014 to 2017. 
		He is the recipient of the 2015 George Axelby Outstanding Paper Award, 2014 CDC Best Student Paper Award (as supervisor), and the 2016 NecSys Best Student Paper Award (as supervisor). His principal research interests include learning and control of centralized and decentralized stochastic systems.
	\end{IEEEbiography}
	\vskip 0pt plus -1fil
	\begin{IEEEbiography}{Stephen L. Smith} (S'05--M'09--SM'15) is an Associate Professor in the department of Electrical and Computer Engineering, University of Waterloo, Canada where he holds a Canada Research Chair in Autonomous Systems.  He is an Associate Editor for the IEEE Transactions on Control of Network Systems, and is the Co-General Chair for the 2021 IEEE International Conference on Robot and Human Interactive Communication (RO-MAN).  His main research interests lie in control and optimization for autonomous systems, with an emphasis on robotic motion planning and coordination.
	\end{IEEEbiography}

	\onecolumn
	\appendix
	For any fixed value of $\lambda$, the value function $V_\lambda(x)$ can also be written as 
	\begin{align*}
		V_\lambda(x) = \min_{\pi\in\Pi} \mathbb{E} \bigg[\sum_{t=0}^T \pig[C(X_t,A_t) + \lambda A_t\pig]|X_0=x \bigg],
	\end{align*}
	where $\Pi$ denotes the set of all Markov policies from $\mathcal{X}$ to $\{0,1\}$. Since the state space $\mathcal{X}$ is finite, so is $\Pi$. Thus, $V_\lambda(x)$ is the minimum of a finite number of functions, each of which is linear in $\lambda$. Therefore, $V_\lambda(x)$ is continuous and piecewise linear, with a finite number of corner points. This means $V_\lambda(x)$, and therefore $Q_\lambda(x,a)$ and $B_\lambda(x)$, are non differential w.r.t. $\lambda$ at a finite number of points. Therefore, $B_\lambda(x)$ is monotonically increasing if $\partial B_\lambda(x)/\partial\lambda$, wherever it exists, is non-negative.
	Let $\Lambda^*(x)$ denote the finite set of values where $B_\lambda(x)$ is non-differentiable. Let $\Lambda^* = \cup_{x\in\mathcal{X}}\Lambda^*(x)$, which is also finite.
	
	The main idea for the proof of Theorem~\ref{th:single_robot_indexability} is to show that if~\eqref{eq:cond_a1_geq0} is satisfied then $\partial B_\lambda(x)/\partial\lambda$ is non-negative for $\lambda\notin\Lambda^*$. Then, Lemma~\ref{th:non-dec_B} implies the indexability of the problem.
	
	Now fix an $n\in\{1,\ldots,N\}$. Define $z = (n, 0)$, $z' = (n+1, 0)$ and $e = (n, 1)$. Then using Eq.~\eqref{eq:Q_value_expanded} and Fig.~\ref{fig:state_transition}, we have
	\begin{align}
		Q_\lambda(z,a) &= C(z,a) + \lambda a + \gamma q^a_{n0} V_\lambda(e) \nonumber\\&\qquad\qquad\;+ \gamma p^a_{n0} V_\lambda(z') + \gamma r^a_{n0} V_\lambda(z).\label{eq:Q_z}\\
		Q_\lambda(e,a) &= C(e,a) + \lambda a + \gamma q^a_{n1} V_\lambda(z) \nonumber\\&\qquad\qquad\;+ \gamma p^a_{n1} V_\lambda(z') + \gamma r^a_{n1} V_\lambda(e).
		\label{eq:Q_e}
	\end{align}
	Then we have the following results:
	
	\section{Preliminary Results}
	\begin{lemma}
		For all $\lambda\notin\Lambda^*$,
		\begin{align*}
			0 \leq \frac{\partial V_\lambda(x)}{\partial \lambda} \leq \frac{1}{1-\gamma}, \quad\forall x \in \mathcal{X}.
		\end{align*}
		\label{th:partialV_bound}
	\end{lemma}
	\begin{proof}
		Under an optimal policy $\pi^*$, we have:
		\begin{equation*}
			V_\lambda(x) = \mathbb{E}\biggl[ \sum_{t = 0}^{\infty} \gamma^t C_\lambda(X_t, \pi^*(X_t)) \bigg| X_0 = x \biggr].
		\end{equation*}
		Therefore, we get
		\begin{align*}
			\frac{\partial V_\lambda(x)}{\partial \lambda} &= \frac{\partial}{\partial \lambda} \mathbb{E}\biggl[ \sum_{t = 0}^{\infty} \gamma^t C_\lambda(X_t, \pi^*(X_t)) \bigg| X_0 = x \biggr] \\
			&= \mathbb{E}\biggl[ \sum_{t = 0}^{\infty} \gamma^t \frac{\partial}{\partial \lambda} C_\lambda(X_t, \pi^*(X_t)) \bigg| X_0 = x \biggr]. 
		\end{align*}
		Since $\frac{\partial}{\partial \lambda} C_\lambda(x, \pi^*(x)) \in [0,1]$ for all $x \in \mathcal{X}$, we can write
		\begin{align*}
			0 &\leq \frac{\partial V_\lambda(x)}{\partial \lambda} \leq \sum_{t = 0}^{\infty} \gamma^t \implies 0 \leq \frac{\partial V_\lambda(x)}{\partial \lambda} \leq \frac{1}{1-\gamma}.
		\end{align*}
	\end{proof}
	
	Define
	\begin{align*}
		\alpha_0(n) &= 1,\\
		\beta_0(n) &= \frac{\gamma (p^1_{n0} - p^0_{n0}) + \gamma^2 {\left(p^0_{n0} r^1_{n0} - p^1_{n0} r^0_{n0} \right)}}{1 - \gamma \,r_{n0}^0},
		\displaybreak[1]\\
		\alpha_1(n) &= 1 + \frac{\gamma q^1_{n0}}{1 - \gamma \,r^1_{n1}} + \frac{\gamma q^0_{n0} \left(\gamma \,r^1_{n0} + \cfrac{\gamma^2 q^1_{n0}q^1_{n1}}{1 - \gamma \,r^1_{n1}} -1\right)}
		{1 - \gamma \,r^1_{n1} -\gamma \,r^0_{n0} +\gamma^2 r^1_{n1} r^0_{n0} -\gamma^2 q^0_{n0} q^1_{n1}},
		\displaybreak[1]\\
		\beta_1(n) &= \gamma p^1_{n0} + \frac{\gamma^2 q^1_{n0}p^1_{n0}}{1 - \gamma r^1_{n1}} + \frac{\left(\gamma p^0_{n0} - \gamma^2 p^0_{n0}r^1_{n1} + \gamma^2 q^0_{n0}p^1_{n0} \right) \left(\gamma r^1_{n0} + \cfrac{\gamma^2 q^1_{n0}q^1_{n1}}{1 - \gamma r^1_{n1}} -1\right)}{1 - \gamma r^1_{n1} -\gamma r^0_{n0} +\gamma^2 r^1_{n1} r^0_{n0} -\gamma^2 q^0_{n0} q^1_{n1}}.
	\end{align*}
	Also define $b_{00}(n)=b_{10}(n)=1$ and
	\begin{align*}
		b_{01}(n) &= \frac{1-\gamma r^0_{n0}}{1-\gamma r^1_{n0}},\\
		b_{11}(n) &= \frac{1 - \gamma r^1_{n1} -\gamma r^0_{n0} +\gamma^2 r^1_{n1} r^0_{n0} -\gamma^2 q^1_{n1} q^0_{n0}}
		{1 - \gamma r^1_{n1} -\gamma r^1_{n0} +\gamma^2 r^1_{n1} r^1_{n0} -\gamma^2 q^1_{n1} q^1_{n0}}.
	\end{align*}
	
	
	%
	\begin{lemma}
		Let $\pi_\lambda(z)=i$ and $\pi_\lambda(e)=j$, then for $\lambda\notin\Lambda^*$,
		\begin{align*}
			\frac{\partial B_\lambda(z)}{\partial\lambda} = b_{ij}(n)\bigg[\alpha_j(n) + \beta_j(n) \frac{\partial V_\lambda(z')}{\partial\lambda}\bigg].
		\end{align*}
		\label{th:partialB_cases}
	\end{lemma}
	\begin{proof}
		The result follows from considering the four cases $(i,j)\in\{(0,0), (0,1), (1,0), (1,1)\}$ separately and simplifying
		\begin{equation}
			\frac{\partial B_\lambda(z)}{\partial \lambda} = \frac{\partial Q_\lambda(z, 1)}{\partial \lambda} - \frac{\partial Q_\lambda(z, 0)}{\partial \lambda}. \label{eq:partial_B}
		\end{equation}
		
		\textbf{Example Case:} $(i,j)=(0,1)$:\\
		Since $\pi_\lambda(z) = 0$, we have $V_\lambda(z) = Q_\lambda(z,0)$. Therefore using \eqref{eq:Q_z}, we get
		\begin{align}
			Q_\lambda(z,0) &= \frac{C(z,0) + \gamma q^0_{n0} V_\lambda(e) + \gamma p^0_{n0} V_\lambda(z')}{1 - \gamma r^0_{n0}}, \label{eq:Q_z0_1}\\
			Q_\lambda(z,1) &= C(z,1) + \lambda + \gamma q^1_{n0} V_\lambda(e) + \gamma p^1_{n0} V_\lambda(z')\nonumber\\ & \qquad\qquad\qquad\qquad+ \gamma r^1_{n0} Q_\lambda(z,0). \label{eq:Q_z1_1}
		\end{align}
		Since $\pi_\lambda(e) = 1$, we have $V_\lambda(e) = Q_\lambda(e,1)$. From \eqref{eq:Q_e}, we get
		\begin{equation}
			V_\lambda(e) = \frac{C(e,1) + \lambda + \gamma q^1_{n1} V_\lambda(z) + \gamma p^1_{n1} V_\lambda(z')}{1-\gamma r^1_{n1}}.
			\label{eq:V_e_active}
		\end{equation}
		Differentiating w.r.t. $\lambda$, we get
		\begin{align}
			\frac{\partial Q_\lambda(z, 0)}{\partial \lambda} &= \frac{1}{1 - \gamma r^0_{n0}} \left({\gamma q^0_{n0} \frac{\partial V_\lambda(e)}{\partial \lambda} + \gamma p^0_{n0} \frac{\partial V_\lambda(z')}{\partial \lambda}}\right),\nonumber\\
			\frac{\partial Q_\lambda(z, 1)}{\partial \lambda} &= 1 + \gamma q^1_{n0} \frac{\partial V_\lambda(e)}{\partial \lambda} + \gamma p^1_{n0} \frac{\partial V_\lambda(z')}{\partial \lambda} + \gamma r^1_{n0} \frac{\partial Q_\lambda(z, 0)}{\partial \lambda}, \nonumber\\
			\frac{\partial V_\lambda(e)}{\partial \lambda} &= \frac{1}{1-\gamma r^1_{n1}}\left({1 + \gamma q^1_{n1} \frac{\partial Q_\lambda(z,0)}{\partial \lambda} + \gamma p^1_{n1} \frac{\partial V_\lambda(z')}{\partial \lambda}}\right).
			\label{eq:dVe_expression_01}
		\end{align}
		Therefore, $\frac{\partial Q_\lambda(z, 0)}{\partial \lambda}$ can be written as:
		\begin{align}
			\frac{\partial Q_\lambda(z, 0)}{\partial \lambda} &= \frac{1}{1 - \gamma r^0_{n0}} \bigg(\frac{\gamma q^0_{n0}}{1-\gamma r^1_{n1}}\bigg({1 + \gamma q^1_{n1} \frac{\partial Q_\lambda(z,0)}{\partial \lambda} + \gamma p^1_{n1} \frac{\partial V_\lambda(z')}{\partial \lambda}}\bigg) + \gamma p^0_{n0} \frac{\partial V_\lambda(z')}{\partial \lambda}\bigg),\nonumber\\
			&= \frac{\gamma q^0_{n0}}{(1 - \gamma r^0_{n0})(1-\gamma r^1_{n1})} + \frac{\gamma^2 q^0_{n0} q^1_{n1}}{(1 - \gamma r^0_{n0})(1-\gamma r^1_{n1})} \frac{\partial Q_\lambda(z,0)}{\partial \lambda} + \frac{\gamma^2 q^0_{n0} p^1_{n1} +\gamma p^0_{n0}(1-\gamma r^1_{n1})}{(1 - \gamma r^0_{n0})(1-\gamma r^1_{n1})} \frac{\partial V_\lambda(z')}{\partial \lambda},\nonumber\\
			\frac{\partial Q_\lambda(z, 0)}{\partial \lambda} &= \frac{1}{(1 - \gamma r^0_{n0})(1-\gamma r^1_{n1})} \bigg(\gamma q^0_{n0} + \gamma^2 q^0_{n0} q^1_{n1} \frac{\partial Q_\lambda(z,0)}{\partial \lambda} + \Big(\gamma^2 q^0_{n0} p^1_{n1} +\gamma p^0_{n0}(1-\gamma r^1_{n1})\Big) \frac{\partial V_\lambda(z')}{\partial \lambda}\bigg),\nonumber\\
			\frac{\partial Q_\lambda(z, 0)}{\partial \lambda} &= \frac{1}{(1 - \gamma r^0_{n0})(1-\gamma r^1_{n1})-\gamma^2 q^0_{n0} q^1_{n1}} \bigg(\gamma q^0_{n0} + \Big(\gamma^2 q^0_{n0} p^1_{n1} +\gamma p^0_{n0}(1-\gamma r^1_{n1})\Big) \frac{\partial V_\lambda(z')}{\partial \lambda}\bigg).
			\label{eq:Q0_expression_01}
		\end{align}
		Substituting the above equations in~\eqref{eq:partial_B}, we get
		\begin{align*}
			\frac{\partial B_\lambda(z)}{\partial \lambda} &= \frac{\partial Q_\lambda(z, 1)}{\partial \lambda} - \frac{\partial Q_\lambda(z, 0)}{\partial \lambda} \\
			&= 1 + \gamma q^1_{n0} \frac{\partial V_\lambda(e)}{\partial \lambda} + \gamma p^1_{n0} \frac{\partial V_\lambda(z')}{\partial \lambda} - (1 - \gamma r^1_{n0}) \frac{\partial Q_\lambda(z, 0)}{\partial \lambda}\\
			&= 1 + \Big(\gamma q^1_{n0} - \frac{1- \gamma r^1_{n0}}{1 - \gamma r^0_{n0}} \gamma q^0_{n0}\Big) \frac{\partial V_\lambda(e)}{\partial \lambda} + \Big(\gamma p^1_{n0} - \frac{1- \gamma r^1_{n0}}{1 - \gamma r^0_{n0}} \gamma p^0_{n0} \Big) \frac{\partial V_\lambda(z')}{\partial \lambda}.
		\end{align*}
		
		Substituting the value of $\frac{\partial V_\lambda(e)}{\partial \lambda}$ using equations~\eqref{eq:dVe_expression_01} and~\eqref{eq:Q0_expression_01}, we get
		\begin{align*}
			\frac{\partial B_\lambda(z)}{\partial \lambda} &= b_{01} \Big(\alpha_1(n) + \beta_1(n) \frac{\partial V_\lambda(z')}{\partial \lambda}\Big).
		\end{align*}
		Results for the remaining cases: $(i,j)\in\{(0,0), (1,0), (1,1)\}$ can be obtained in a similar way as above.
	\end{proof}
	
	\begin{lemma}
		For all $\lambda\notin\Lambda^*$, $\partial B_\lambda(e)/\partial \lambda \geq 0$.
		\label{th:conditions_at_error_state}
	\end{lemma}
	\begin{proof}
		We consider two cases:
		
		\textbf{1) Case I: $\pi_\lambda(e)=0$}: \\ 
		From~\eqref{eq:Q_e}, we have 
		\[Q_\lambda(e,0) = V_\lambda(e) = \frac{C(e,0)}{1-\gamma}, \]
		which is independent of $\lambda$.
		Therefore, we get
		\begin{align*}
			\frac{\partial B_\lambda(e)}{\partial \lambda} = \frac{\partial Q_\lambda(e,1)}{\partial \lambda} =  1 + \gamma~q^1_{n1}~\frac{\partial V_\lambda(z)}{\partial \lambda}
			+ \gamma~p^1_{n1}~\frac{\partial V_\lambda(z')}{\partial \lambda}.
		\end{align*}
		From Lemma~\ref{th:partialV_bound}, \blu{we get that $\frac{\partial V_\lambda(z)}{\partial \lambda} \geq 0$ and $\frac{\partial V_\lambda(z')}{\partial \lambda} \geq 0$.} This gives us $\partial B_\lambda/\partial\lambda \geq 0$.
		
		\textbf{2) Case II: $\pi_\lambda(e)=1$}: \\ 
		As a result, we have
		\begin{align*}
			Q_\lambda(e,0) = C(e,0) + \gamma~V_\lambda(e).
		\end{align*}
		Therefore, using~\eqref{eq:partial_B}, we get
		\begin{align*}
			\frac{\partial B_\lambda(e)}{\partial \lambda} = (1-\gamma) \frac{\partial V_\lambda(e)}{\partial \lambda}.
		\end{align*}
		From Lemma~\ref{th:partialV_bound}, we get that $\partial B_\lambda/\partial\lambda \geq 0$.
	\end{proof}

	\subsection{Proof of Theorem~\ref{th:single_robot_indexability}}
	\blu{Lemma~\ref{th:conditions_at_error_state} shows that the Benefit function $B_\lambda(x)$ is always monotonically increasing for all fault states.}
	
	\blu{Since $q^1_{n1} \leq 1 - r^1_{n1}$ and $q^0_{n0} \leq 1 - r^0_{n0}$, we can write
		\begin{equation}
			\gamma^2 q^1_{n1} q^0_{n0} \leq (\gamma - \gamma r^1_{n1})(\gamma - \gamma r^0_{n0}).
			\label{eq:b11_proof1}
		\end{equation}
		Moreover, since $\gamma \in (0,1)$, we have
		\begin{equation}
			(\gamma - \gamma r^1_{n1})(\gamma - \gamma r^0_{n0}) < (1 - \gamma r^1_{n1})(1 - \gamma r^0_{n0}).
			\label{eq:b11_proof2}
		\end{equation}
		From \eqref{eq:b11_proof1} and \eqref{eq:b11_proof2}, we can write
		\begin{equation*}
			1 - \gamma r^1_{n1} -\gamma r^0_{n0} +\gamma^2 r^1_{n1} r^0_{n0} -\gamma^2 q^1_{n1} q^0_{n0} > 0.
		\end{equation*}
		Similarly, 
		\begin{equation*}
			1 - \gamma r^1_{n1} -\gamma r^1_{n0} +\gamma^2 r^1_{n1} r^1_{n0} -\gamma^2 q^1_{n1} q^1_{n0} > 0.
		\end{equation*}
		This implies that $b_{11}> 0$.}
	
	\blu{Lemma~\ref{th:partialV_bound} gives us bounds on the derivative of the value function w.r.t. $\lambda$. These bounds are then used with results of Lemma~\ref{th:partialB_cases} to show that the function $B_\lambda(x)$ is monotonically increasing} for all $x\in\mathcal{X}$ if for all $n \in \{1,2,\ldots,N\}$ and $ j\in\{0,1\}$
	\footnote{\blu{To obtain this result, we make use of Lemma~\ref{th:partialV_bound}, which establishes bounds on the derivative of the value function. Since transition probabilities at each state are independent, monotonicity of the Benefit function $B_\lambda$ is guaranteed if $\alpha_j(n) + \beta_j(n) {\partial V(n+1)}/{\partial \lambda}$ is non-negative for both lower and upper bounds of ${\partial V(n+1)}/{\partial \lambda}$.}}:
	\begin{align*}
		\alpha_j(n) \geq 0 \;\text{and}\; \alpha_j(n) + \beta_j(n)\frac{1}{1-\gamma} \geq 0.
	\end{align*}
	\blu{We observe that $\alpha_0(n) = 1 \geq 0$. Also, $\alpha_1(n)$ and $\beta_1(n)$ can be simplified as:} 
	
	
	\begin{align*}
		\alpha_1(n) &= 1 + \frac{\gamma q^1_{n0}}{1 - \gamma \,r^1_{n1}}  + \frac{\gamma q^0_{n0} \left(\gamma \,r^1_{n0} + \cfrac{\gamma^2 q^1_{n0}q^1_{n1}}{1 - \gamma \,r^1_{n1}} -1\right)}
		{1 - \gamma \,r^1_{n1} -\gamma \,r^0_{n0} +\gamma^2 r^1_{n1} r^0_{n0} -\gamma^2 q^0_{n0} q^1_{n1}} \\[\verticaldistance]
		&= \frac{1 - \gamma \,q_{n0}^0 +\gamma \,q_{n0}^1 -\gamma \,r_{n0}^0 -\gamma \,r_{n1}^1 -\gamma^2 \,q_{n0}^0 \,q_{n1}^1 +\gamma^2 \,q_{n0}^0 \,r_{n0}^1 -\gamma^2 \,q_{n0}^1 \,r_{n0}^0 +\gamma^2 \,r_{n0}^0 \,r_{n1}^1}{1 - \gamma \,r_{n0}^0 -\gamma \,r_{n1}^1 +\gamma^2 \,r_{n0}^0 \,r_{n1}^1 -\gamma^2 \,q_{n0}^0 \,q_{n1}^1},\\[\verticaldistance]
		\beta_1(n) &= \gamma p^1_{n0} + \frac{\gamma^2 q^1_{n0}p^1_{n0}}{1 - \gamma r^1_{n1}} + \frac{\left(\gamma p^0_{n0} - \gamma^2 p^0_{n0}r^1_{n1} + \gamma^2 q^0_{n0}p^1_{n0} \right) \left(\gamma r^1_{n0} + \cfrac{\gamma^2 q^1_{n0}q^1_{n1}}{1 - \gamma r^1_{n1}} -1\right)}{1 - \gamma r^1_{n1} -\gamma r^0_{n0} +\gamma^2 r^1_{n1} r^0_{n0} -\gamma^2 q^0_{n0} p^1_{n1}} \\[\verticaldistance]
		&= \frac{\gamma \,{\left(1 - \gamma\right)}\,{\left(q_{n0}^0 -q_{n0}^1 +r_{n0}^0 -r_{n0}^1 +\gamma \,q_{n0}^0 \,q_{n1}^1 -\gamma \,q_{n0}^1 \,q_{n1}^1 -\gamma \,q_{n0}^0 \,r_{n0}^1 +\gamma \,q_{n0}^1 \,r_{n0}^0 -\gamma \,r_{n0}^0 \,r_{n1}^1 +\gamma \,r_{n0}^1 \,r_{n1}^1 \right)}}{1 - \gamma \,r_{n0}^0 -\gamma \,r_{n1}^1 +\gamma^2 \,r_{n0}^0 \,r_{n1}^1 -\gamma^2 \,q_{n0}^0 \,q_{n1}^1}
	\end{align*}
	Therefore, 
	\begin{align*}
		\alpha_1(n) + \frac{\beta_1(n)}{1-\gamma} = \frac{1 - \gamma r^1_{n0} -\gamma r^1_{n1} +\gamma^2 r^1_{n0} r^1_{n1} -\gamma^2 q^1_{n0} q^1_{n1}}{1 - \gamma r_{n0}^0 -\gamma r_{n1}^1 +\gamma^2 r_{n0}^0 r_{n1}^1 -\gamma^2 q_{n0}^0 q_{n1}^1} = \frac{1}{b_{11}}.
	\end{align*}
	Since, $b_{11}(n) > 0$, we get $\alpha_1(n) + \beta_1(n)/(1-\gamma) > 0$.
	
	
	Therefore, the single-robot problem is indexable if:
	\begin{align*}
		\alpha_1(n) &\geq 0 \quad \text{and} \quad \frac{\beta_0(n)}{1-\gamma} \geq -1, ~ \forall n \in \{1,\ldots,N\}.
	\end{align*}
	\hfill$\qedsymbol$

	
	
	
\end{document}